%% file: iclr2027_conference.tex
\documentclass[11pt]{article}
\usepackage[
    top=0.8in,
    bottom=0.8in,
    left=0.8in,
    right=0.8in
]{geometry}

\usepackage{amsmath}
\usepackage{amssymb}
\usepackage{amsfonts}
\usepackage{subcaption}
\usepackage{graphicx}
\usepackage{booktabs}
\usepackage{multirow}
\usepackage{titletoc}

\usepackage{algorithm}
\usepackage{algorithmic}

\usepackage{xcolor}
\usepackage{minitoc}
\usepackage[numbers]{natbib}
\usepackage{hyperref}

\input{math_commands.tex}

\title{\textbf{Learning Operators for Symmetry Discovery: From Lie Groups to Local Symmetry Geometry}}

\author{
\textbf{Pavan Karjol$^1$ \quad
Scindhiya Laxmi$^2$ \quad
Prathosh AP$^1$}\\[4pt]
$^1$Department of Electrical Communication Engineering, 
Indian Institute of Science, Bangalore, India\\
$^2$Department of Mathematics and Computing, 
IIT (ISM) Dhanbad, India
}

\begin{document}

\maketitle

\begin{abstract}
Continuous symmetry discovery is typically posed as a separate estimation or
optimization problem for each new function or dataset. We instead ask whether
the map from a function to its symmetry structure can be inferred across a
family of functions from finite observations. Continuous symmetries admit an
infinitesimal differential characterization, and can therefore be expressed
through differential constraints on the function. This suggests an inverse
analogue of PDE solving: rather than starting from a differential operator and
solving for a compatible function, we observe the function and infer the
infinitesimal differential constraints that encode its symmetry. We formulate
this map as an inverse symmetry operator.

For global Lie-group symmetries, we show that this inverse map admits a
representation-theoretic population characterization. After decomposing the
function into representation coefficients, its symmetry Lie algebra is the
nullspace of a positive-semidefinite operator constructed from the
infinitesimal group action. The inverse symmetry operator is equivariant under
changes of coordinates through the corresponding adjoint action.

This construction separates the geometric symmetry map from the finite-sample
problem of estimating the quantities required to evaluate it. We therefore
amortize only those parts of the inference pipeline that benefit from
information shared across functions, while retaining analytic structure where
it is available. On challenging $\SO(5)$ benchmarks, amortized equivariant estimation of harmonic coefficients improves symmetry recovery from sparse
observations, whereas for $\SL(2,\mathbb R)$ ordinary least squares is already
sufficient.

We further extend the same viewpoint beyond global symmetries. When symmetry directions vary across the domain, they define a smoothly varying rank-\(k\) field of tangent subspaces, or equivalently a rank-$k$ tangent distribution, rather than a single fixed Lie algebra. If these local directions satisfy the appropriate Lie-bracket consistency condition, Frobenius' theorem guarantees that they locally form the tangent spaces of \(k\)-dimensional integral submanifolds along which the observables remain unchanged. We recover this geometry from finite samples using a hybrid estimator that learns the local weighting of observations while keeping the differential and geometric recovery steps analytic.

Across global, non-compact, local, and physically motivated synthetic
benchmarks, the results support a common principle: learn the statistically
difficult finite-sample component while preserving the known geometric map to
the symmetry object.
\end{abstract}

\input{Introduction}

\section{Related Work}
\label{rel_work}

\subsection{Continuous Symmetry Discovery}

A growing body of work learns continuous symmetries directly from data by
estimating infinitesimal generators or transformation parameters
\citep{zhou2020meta,dehmamy2021automatic,moskalev2022liegg,
yang2023generative,hu2025explicit}. These methods typically formulate symmetry
discovery as an optimization problem for a particular dataset or system. For
example, LieGAN learns Lie-algebra generators adversarially so that transformed
samples preserve the observed data distribution
\citep{yang2023generative}. Such approaches are well suited to discovering the
symmetry of an individual instance, but applying them across a family of
functions generally requires repeating the symmetry-estimation procedure for
each new realization. Our work instead asks whether information can be shared
across such a family by treating symmetry discovery itself as an inverse
operator.

\subsection{Operator Learning for Functional Ensembles}

Neural operators, including Fourier Neural Operators and DeepONets, learn maps
between function spaces and have proved effective for families of parametric
PDEs \citep{li2020fourier,lu2021learning,kovachki2023neural}. Their central
advantage is that inference for a new instance can reuse structure learned
across related functions rather than solving each problem independently.
Existing operator-learning methods, however, primarily predict scalar- or
vector-valued functions. Our output is instead a geometric object, such as a
Lie subalgebra or a field of tangent subspaces, whose transformation laws and
algebraic structure must be respected. We therefore use operator learning as a
viewpoint for amortizing the difficult finite-sample inference step while
retaining the known geometric map to the symmetry object.

\subsection{Local Symmetry and Nonlinear Generators}

Beyond global affine or linear transformations, recent methods have considered
nonlinear and spatially varying symmetry structure. Beyond-Affine
\citep{shaw2024symmetry} estimates tangent vector fields that annihilate
task-relevant functions and thereby represents continuous symmetries beyond the
affine group. This makes it a natural quantitative baseline for our local
experiments: its learned vector fields can be evaluated pointwise and compared
directly with the target tangent space
$\mathcal D_F(x)=\ker J_F(x)$.

AtlasD \citep{pmlr-v267-bhat25a} also studies local symmetry, but adopts a
different notion of locality. It formalizes local symmetry through atlas
equivariance, training localized predictors and discovering a Lie-group basis
under which those predictors are equivariant. The resulting object is therefore
an atlas-level symmetry group rather than a query-dependent tangent
distribution. Converting its fixed chart-coordinate generators into a
continuously varying field $x\mapsto\mathcal D_F(x)$ would require modifying
the published method, for example through chart-specific or nonlinear
generators. We therefore treat AtlasD as closely related work rather than as a
direct pointwise baseline.

Our local formulation instead targets the tangent distribution itself. When
this distribution has constant rank and is closed under Lie brackets,
Frobenius' theorem guarantees local integrability into submanifolds along which
the observables remain unchanged.

\input{proposed-method}

\input{Experiments-v1}
\vspace{-0.2cm}

\subsection{Limitations and Future Work}
\vspace{-0.2cm}

Our global construction assumes a representation in which the infinitesimal
group action is tractable; with truncated representations, recovery is exact
only when the retained coefficients distinguish the symmetry. The local method
also assumes a constant-rank regime, excluding singular points where symmetry
dimension changes. Finally, our physical experiments use synthetic systems
with known ground truth. Future work includes learned representations,
uncertainty-aware dimension selection, singular or stratified local symmetries,
higher-order jet constraints, and validation on real scientific data.
\vspace{-0.2cm}
\section{Conclusion}
\vspace{-0.2cm}
We formulated continuous symmetry discovery as an inverse operator problem.
For global symmetries, the symmetry Lie algebra is recovered analytically as a
stabilizer nullspace; for local symmetries, the target becomes a field of
tangent subspaces obtained from Jacobian nullspaces. Across our experiments,
the main lesson is consistent: learn the difficult finite-sample component,
while preserving the known geometric map to the symmetry structure.




\subsection*{AI use statement}

Generative AI and large language models were used to assist with literature
search and organization, coding and implementation, and editing and improving
the clarity of the manuscript. All literature claims and citations were checked
against the corresponding sources, and AI-assisted code was reviewed, tested,
and validated through the experiments and numerical checks reported in the
paper. All mathematical arguments, experimental design choices, results, and
final claims were reviewed by the authors. The authors take responsibility for
the complete content of the paper, including material produced with AI
assistance.

\subsection*{Ethics statement}

This work is methodological and uses synthetic datasets and simulated physical
benchmarks with analytically controlled symmetries. It does not involve human
subjects, personal data, or sensitive datasets. We are not aware of specific
ethical risks arising from the experiments considered here. As with other
methods for automated scientific inference, conclusions drawn from symmetry
discovery on real data should nevertheless account for modeling assumptions,
measurement noise, and possible misspecification.

\subsection*{Reproducibility statement}

We provide the mathematical assumptions and proofs in the appendix, together
with detailed descriptions of the benchmark construction, model architectures,
training procedures, optimization settings, evaluation metrics, data splits,
random seeds, and baseline implementations. Train, validation, development,
and final-test functions are generated from disjoint seeds, and model and
hyperparameter selection is performed without access to the final-test set.
Additional numerical results, ablations, and implementation diagnostics are
included in the supplementary material.

\bibliography{iclr2027_conference}
\bibliographystyle{iclr2027_conference}

\clearpage
\appendix

\section*{Appendix}

\startcontents[appendix]
\printcontents[appendix]{}{0}{}

\section{Notation}
\label{sec:notation}
For clarity and consistency, the main mathematical notation used throughout the manuscript is summarized in Table~\ref{tab:notation}.

\begin{table}[!t]
\centering
\caption{Summary of the notation used throughout the manuscript.}
\label{tab:notation}
\small
\begin{tabular}{ll}
\hline
\textbf{Notation} & \textbf{Description} \\
\hline
$G$ & Continuous transformation group acting on the domain \\
$\mathfrak{g}$ & Lie algebra of $G$ \\
$X$ & Domain/manifold on which $G$ acts \\
$x$ & Point in the domain $X$ \\
$g$ & Element of the transformation group $G$ \\
$f$ & Scalar-valued function whose symmetries are to be discovered \\
$H_f$ & Stabilizer subgroup of $f$ under the action of $G$ \\
$\mathfrak{h}_f$ & Lie algebra of continuous symmetries of $f$ \\
$\dim(\mathfrak{h}_f)$ & Dimension of the symmetry Lie algebra \\
$\operatorname{Ad}_g$ & Adjoint action of $g$ on $\mathfrak{g}$ \\
$\rho_{\lambda}$ & Representation associated with representation block $\lambda$ \\
$d\rho_{\lambda}$ & Derived representation acting on the Lie algebra \\
$V_{\lambda}$ & Representation space associated with $\lambda$ \\
$\Lambda$ & Set of representation/harmonic blocks used in the construction \\
$c_\lambda$ & Coefficient vector of $f$ in representation block $\lambda$ \\
$J_\lambda(c_\lambda)$ & Infinitesimal symmetry operator for block $\lambda$ \\
$J_\lambda^*$ & Adjoint of $J_\lambda$ \\
$K_\lambda$ & Block-wise positive-semidefinite stabilizer operator \\
$K_f$ & Global stabilizer operator, $\displaystyle K_f=\sum_{\lambda\in\Lambda}J_\lambda^*J_\lambda$ \\
$\ker K_f$ & Symmetry Lie algebra $\mathfrak{h}_f$ \\
$\mathcal{S}$ & Exact inverse symmetry operator, $f\mapsto\mathfrak{h}*f$ \\
$\mathcal{S}_{\theta}$ & Learned/amortized approximation of the inverse symmetry operator \\
$\theta$ & Trainable parameters of the learned symmetry operator \\
$\mathcal{D}_f$ & Finite set of observations/samples of $f$ \\
$\widehat{\mathfrak{h}}_f$ & Estimated symmetry Lie algebra from finite observations \\
$P_f$ & Orthogonal projector onto the true symmetry subspace $\mathfrak{h}_f$ \\
$\widehat{P}_f$ & Estimated symmetry projector \\
$d_f$ & True symmetry dimension, $d_f=\dim(\mathfrak{h}_f)$ \\
$\widehat d_f$ & Estimated symmetry dimension \\
$\mathcal{L}_v f$ & Lie derivative of $f$ along vector field $v$ \\
$v_{\xi}$ & Infinitesimal vector field generated by $\xi\in\mathfrak{g}$ \\
$A$ & Matrix representation of an infinitesimal linear transformation \\
$\mathcal{D}_F(x)$ & Local symmetry distribution at $x$, $\mathcal{D}*F(x)=\ker dF_x$ \\
$F$ & Vector-valued map used for local symmetry discovery \\
$dF_x$ & Differential (Jacobian map) of $F$ at $x$ \\
$J_F(x)$ & Jacobian matrix of $F$ at $x$ \\
$k$ & Dimension of the local symmetry distribution \\
$\mathcal{W}_{\theta}$ & Learned local symmetry/distribution estimator \\
$[X,Y]$ & Lie bracket of vector fields $X$ and $Y$ \\
$\operatorname{rank}dF_x$ & Local rank of $F$ at $x$ \\
$\mathrm{SO}(n)$ & Special orthogonal group in $n$ dimensions \\
$\mathrm{SL}(2,\mathbb{R})$ & Special linear group of $2\times2$ real matrices \\
\hline
\end{tabular}
\end{table}

\section{Mathematical Background}
\label{app:mathematical_background}

\subsection{Lie Groups and Lie Algebras}
\label{app:lie_groups}
A \emph{Lie group} is a smooth manifold $G$ equipped with a group structure
such that the group multiplication and inversion maps are smooth. Lie groups
provide a natural mathematical framework for continuous transformations such
as rotations, translations, and changes of coordinates. In this work, $G$
denotes a Lie group acting smoothly on an input space $X$.

Associated with a Lie group $G$ is its \emph{Lie algebra} $\mathfrak{g}$,
which is the tangent space of $G$ at the identity element $e$, $\mathfrak{g}=T_eG.$ 
Although elements of $G$ describe finite transformations, elements
$\xi\in\mathfrak{g}$ describe infinitesimal transformations. The Lie algebra
is equipped with the Lie bracket $[\xi,\eta],~ \xi,\eta\in\mathfrak{g},$ 
which characterizes the infinitesimal non-commutativity of the corresponding
transformations. A linear subspace $\mathfrak{h}\subseteq\mathfrak{g}$ satisfying
$[\mathfrak{h},\mathfrak{h}]\subseteq\mathfrak{h}$ is called a \emph{Lie subalgebra}.

In our setting, the unknown symmetry structure of a function is represented by
such a Lie subalgebra rather than by an arbitrary collection of transformations.

\subsection{Group Actions}
\label{app:group_actions}
Let $G$ be a Lie group and $X$ a smooth manifold. A \emph{smooth group
action} of $G$ on $X$ is a smooth map $G\times X\to X, ~(g,x)\mapsto g\cdot x,$ satisfying
$e\cdot x=x, ~
(g_1g_2)\cdot x
=
g_1\cdot(g_2\cdot x),
$
for all $g_1,g_2\in G$ and $x\in X$. A group action describes how transformations in $G$ modify points in the domain. 

It also induces an action on functions $f:X\to\mathbb{R}$ through the
pullback $(g\cdot f)(x)
=
f(g^{-1}\cdot x).$ With this convention, a function is invariant under a transformation $g$ if
$
g\cdot f=f,
$
or equivalently,
$
f(g\cdot x)=f(x)
\qquad
\text{for all }x\in X.
$ The induced action on functions is the basis for defining the symmetry
structure that we seek to recover from finite observations.

\subsection{Stabilizer and Symmetry Lie Algebra}
\label{app:stabilizer}

For a function $f:X\to\mathbb{R}$, its \emph{stabilizer} (or isotropy subgroup)
under the action of $G$ is $H_f
=
\{g\in G:\;g\cdot f=f\}.
$ Thus, $H_f$ consists of all transformations that leave $f$ unchanged. The
continuous component of this symmetry is described by the Lie algebra $
\mathfrak{h}_f
=
\operatorname{Lie}(H_f)
\subseteq\mathfrak{g}.
$  The dimension $
d_f=\dim(\mathfrak{h}_f)
$
therefore measures the number of independent infinitesimal symmetry
directions. This is the primary global symmetry object recovered by our
inverse symmetry operator,
\[
\mathcal{S}:f\mapsto\mathfrak{h}_f.
\]

\subsection{Infinitesimal Actions and Lie Derivatives}
\label{app:infinitesimal_symmetry}

An element $\xi\in\mathfrak{g}$ generates a one-parameter subgroup of $G$,
and its action on $X$ induces a vector field $v_\xi$ on the domain. The
corresponding infinitesimal change of a function $f$ is described by the Lie
derivative $\mathcal{L}_{v_\xi}f.$ A generator $\xi$ belongs to the symmetry Lie algebra if and only if the  function is invariant along the corresponding infinitesimal flow:
$$
\xi\in\mathfrak{h}_f
\quad\Longleftrightarrow\quad
\mathcal{L}_{v_\xi}f=0.
$$
For a linear action on $X\subseteq\mathbb{R}^n$, an infinitesimal generator can
be represented by a matrix $A$, with the induced vector field $v_A(x)=Ax.$ 
The infinitesimal invariance condition then becomes
\[
\mathcal{L}_{v_A}f(x)
=
\nabla f(x)^{\top}Ax
=
0,
\qquad
x\in X.
\]
Thus, symmetry discovery can be viewed as an inverse problem: given observations
of $f$, recover the infinitesimal generators that satisfy this differential
constraint.

\subsection{Adjoint Action and Equivariance}
\label{app:adjoint_action}

The Lie algebra transforms naturally under changes of coordinates induced by
elements of the group. For $g\in G$, the \emph{adjoint action} is the map
$\operatorname{Ad}_g:\mathfrak{g}\to\mathfrak{g},$ defined by differentiating conjugation,
$$
\operatorname{Ad}_g(\xi)
=
\left.
\frac{d}{dt}
\left(
g\exp(t\xi)g^{-1}
\right)
\right|_{t=0}.
$$

Under the pullback action on functions, the stabilizer is conjugated by $g$.
Consequently, the symmetry Lie algebra transforms according to $\mathcal{S}(g\cdot f)
=
\operatorname{Ad}_g\,\mathcal{S}(f).
$ Hence, the inverse symmetry operator is naturally equivariant: transforming
the input function transports its symmetry algebra according to the adjoint
action rather than leaving it fixed.

\subsection{Representation-Theoretic Description}
\label{app:representations}

Let the function space admit a decomposition into representation spaces
$V_\lambda$ associated with the action of $G$. We write the representation
coefficients of a function schematically as
$
f
\longleftrightarrow
\{c_\lambda\}_{\lambda\in\Lambda},
~
c_\lambda\in V_\lambda.
$

The infinitesimal action of $\xi\in\mathfrak{g}$ on a coefficient block is
given by the derived representation $d\rho_\lambda(\xi)c_\lambda.$
This defines the linear map $J_\lambda(c_\lambda):
\mathfrak{g}\to V_\lambda,
~
J_\lambda(c_\lambda)\xi
=
d\rho_\lambda(\xi)c_\lambda.
$ An infinitesimal generator is a symmetry of $f$ precisely when it stabilizes
every representation component. This motivates the positive-semidefinite
operator
\[
K_f
=
\sum_{\lambda\in\Lambda}
J_\lambda(c_\lambda)^*
J_\lambda(c_\lambda).
\]
Its nullspace characterizes the symmetry Lie algebra:
\[
\ker K_f
=
\bigcap_{\lambda\in\Lambda}
\ker J_\lambda(c_\lambda)
=
\mathfrak{h}_f.
\]
Thus, the geometric problem of identifying a Lie subalgebra can be reduced to
a finite-dimensional nullspace problem on $\mathfrak{g}$.

\subsection{Local Symmetry Distributions}
\label{app:local_distributions}

Global symmetry is described by a fixed Lie subalgebra
$\mathfrak{h}_f\subseteq\mathfrak{g}$. In a more general setting, the
directions along which observables remain locally unchanged may vary with
position. This motivates the notion of a \emph{distribution}.

Let $M$ be a smooth manifold and consider a vector-valued observable
$
F=(f_1,\ldots,f_m):M\to\mathbb{R}^m.
$
At a regular point $x\in M$, the local directions along which all components
remain unchanged to first order are
\[
\mathcal{D}_F(x)
=
\bigcap_{r=1}^{m}\ker df_r(x)
=
\ker dF_x,
\]
where
$
dF_x:T_xM\to\mathbb{R}^m
$
is the differential of $F$ at $x$. In local coordinates, this can be written as
$
\mathcal{D}_F(x)=\ker J_F(x),
$
where $J_F(x)$ is the Jacobian matrix of $F$. If $\operatorname{rank}dF_x=n-k,$ for an $n$-dimensional manifold, then $\dim\mathcal{D}_F(x)=k.$

This provides the local analogue of the global infinitesimal symmetry
condition. Instead of recovering one fixed Lie algebra,
$f\mapsto\mathfrak{h}_f,$ the local problem recovers a field of tangent subspaces, $F\mapsto\mathcal{D}_F(\cdot).$

\subsection{Involutivity and Frobenius Integrability}
\label{app:frobenius}

A distribution $\mathcal{D}$ is called \emph{involutive} if for any smooth
vector fields $X$ and $Y$ taking values in $\mathcal{D}$, their Lie bracket
also belongs to $\mathcal{D}$: $
X,Y\in\mathcal{D}
\Longrightarrow
[X,Y]\in\mathcal{D}.
$ For the common-kernel distribution defined above, suppose
$Xf_r=Yf_r=0
~
\text{for all }r.
$ Then
$
[X,Y]f_r
=
X(Yf_r)-Y(Xf_r)
=
0.
$ Hence $[X,Y]$ also belongs to the common kernel. Under the usual smoothness
and constant-rank assumptions, $\mathcal{D}_F$ is therefore involutive.

By the \emph{Frobenius theorem}, a regular involutive distribution is locally
integrable: through every point there exists a local submanifold whose tangent
space coincides with the distribution. Thus, the local nullspace construction
used in our framework has a corresponding geometric interpretation in terms
of locally invariant directions and integral manifolds.

\section{Proofs}
\label{sec:proofs}

\begin{proposition}
Let
$f\leftrightarrow\{c_\lambda\}_{\lambda\in\Lambda}$ be a
$G$-equivariant representation decomposition and define
\[
K_f=
\sum_{\lambda\in\Lambda}
J_\lambda(c_\lambda)^*J_\lambda(c_\lambda),
\qquad
J_\lambda(c_\lambda)\xi
=
d\rho_\lambda(\xi)c_\lambda.
\]
Then
\[
\ker K_f
=
\bigcap_{\lambda\in\Lambda}
\ker J_\lambda(c_\lambda)
=
\mathfrak h_f.
\]
\end{proposition}

\begin{proof}
We first show that
\[
\ker K_f
=
\bigcap_{\lambda\in\Lambda}
\ker J_\lambda(c_\lambda).
\]
For any $\xi\in\mathfrak{g}$,
\[
\begin{aligned}
\left\langle K_f\xi,\xi\right\rangle
&=
\sum_{\lambda\in\Lambda}
\left\langle
J_\lambda(c_\lambda)^*
J_\lambda(c_\lambda)\xi,\xi
\right\rangle \\
&=
\sum_{\lambda\in\Lambda}
\left\|
J_\lambda(c_\lambda)\xi
\right\|_{V_\lambda}^{2}.
\end{aligned}
\]
Every term in the sum is nonnegative. Hence,
\[
K_f\xi=0
\quad\Longleftrightarrow\quad
\left\langle K_f\xi,\xi\right\rangle=0
\quad\Longleftrightarrow\quad
J_\lambda(c_\lambda)\xi=0
\quad
\text{for every }\lambda\in\Lambda.
\]
Therefore,
\[
\ker K_f
=
\bigcap_{\lambda\in\Lambda}
\ker J_\lambda(c_\lambda).
\]

It remains to identify this common kernel with $\mathfrak h_f$. By definition, $\xi\in\mathfrak h_f
\quad\Longleftrightarrow\quad
\xi
\text{ generates an infinitesimal transformation that leaves }f
\text{ invariant}.
$
Equivalently, $d\rho(\xi)f=0,$ 
where $d\rho$ denotes the infinitesimal representation induced by the
$G$-action on the function space. Since the representation decomposition is
$G$-equivariant, $f\leftrightarrow\{c_\lambda\}_{\lambda\in\Lambda}$
and the infinitesimal action decomposes blockwise as 
$$d\rho(\xi)f
\leftrightarrow
\left\{
d\rho_\lambda(\xi)c_\lambda
\right\}_{\lambda\in\Lambda}.$$
Consequently,
\[
d\rho(\xi)f=0
\quad\Longleftrightarrow\quad
d\rho_\lambda(\xi)c_\lambda=0
\quad
\text{for every }\lambda\in\Lambda.
\]
Using the definition of $J_\lambda$, $
d\rho_\lambda(\xi)c_\lambda
=
J_\lambda(c_\lambda)\xi,$
and hence
$
\xi\in\mathfrak h_f
\Longleftrightarrow
\xi\in
\bigcap_{\lambda\in\Lambda}
\ker J_\lambda(c_\lambda).
$
Thus,
\[
\bigcap_{\lambda\in\Lambda}
\ker J_\lambda(c_\lambda)
=
\mathfrak h_f.
\]
Combining the two equalities gives
\[
\boxed{
\ker K_f
=
\bigcap_{\lambda\in\Lambda}
\ker J_\lambda(c_\lambda)
=
\mathfrak h_f
}.
\]
\end{proof}

\begin{proposition}[Equivariance of the inverse symmetry operator]
For every $g\in G$,
\[
\mathcal S(g\cdot f)
=
\operatorname{Ad}_g\,\mathcal S(f).
\]
\end{proposition}

\begin{proof}
Recall that the action of $G$ on functions is given by the pullback
\[
(g\cdot f)(x)=f(g^{-1}\cdot x),
\]
and define the stabilizer of $f$ by
\[
H_f=\{h\in G:h\cdot f=f\}.
\]
Let
\[
f'=g\cdot f.
\]
We first determine the stabilizer of $f'$. An element $h\in G$ belongs to
$H_{f'}$ if and only if
\[
h\cdot f'=f'.
\]
Using $f'=g\cdot f$, this is equivalent to
\[
h\cdot(g\cdot f)=g\cdot f.
\]
Since the action is a group action,
\[
(hg)\cdot f=g\cdot f.
\]
Applying $g^{-1}$ to both sides gives
\[
(g^{-1}hg)\cdot f=f.
\]
Therefore,
\[
g^{-1}hg\in H_f,
\]
or equivalently,
\[
h\in gH_fg^{-1}.
\]
Hence the stabilizers are related by conjugation:
\[
H_{g\cdot f}=gH_fg^{-1}.
\]

Taking Lie algebras and using the standard relation between conjugation and
the adjoint action gives
\[
\mathfrak h_{g\cdot f}
=
\operatorname{Lie}(H_{g\cdot f})
=
\operatorname{Lie}(gH_fg^{-1})
=
\operatorname{Ad}_g\mathfrak h_f.
\]
Since the inverse symmetry operator is defined by
\[
\mathcal S(f)=\mathfrak h_f,
\]
we obtain
\[
\boxed{
\mathcal S(g\cdot f)
=
\operatorname{Ad}_g\,\mathcal S(f)
}.
\]
Thus, the inverse symmetry operator is equivariant with respect to the
pullback action on functions and the adjoint action on the Lie algebra.
\end{proof}

\input{supplementary}





\end{document}

%% file: math_commands.tex
\usepackage{amsthm}

\newtheorem{theorem}{Theorem}[section]
\newtheorem{proposition}[theorem]{Proposition}

\usepackage{amsmath,amsfonts,bm}

\providecommand{\SO}{\mathrm{SO}}
\providecommand{\SL}{\mathrm{SL}}

\def\eqref#1{equation~\ref{#1}}

\def\1{\bm{1}}

\DeclareMathAlphabet{\mathsfit}{\encodingdefault}{\sfdefault}{m}{sl}
\SetMathAlphabet{\mathsfit}{bold}{\encodingdefault}{\sfdefault}{bx}{n}



%% file: Introduction.tex
\section{INTRODUCTION}
Deep learning models and scientific computing increasingly exploit continuous
symmetries to build representations that are both sample efficient and
physically consistent. Such symmetries encode structural invariances, including
rotations and changes of frame, that often reflect the underlying organization
of physical systems \citep{wang2020incorporating}. In PDEs and dynamical
systems, Lie-group symmetries support dimensional reduction, conserved
quantities, and exact integration techniques
\citep{olver1993applications,sjoberg2007double,brandstetter2022lie}.
Automatically discovering these symmetries from data is therefore an important
problem in scientific machine learning \citep{dehmamy2021automatic}.

Many recent symmetry-discovery methods, however, formulate inference
instance by instance, optimizing generators or transformation parameters
separately for each observed system. This becomes inefficient when one must
analyze a family of related functions, as in parametric PDEs with varying
coefficients, parameters, or boundary conditions \citep{gabel2024data}.
A new function typically requires a new symmetry-estimation procedure rather
than reusing information learned from previous instances. This motivates a
different question: can symmetry discovery itself be treated as a reusable map
from a sampled function to its continuous symmetry structure?

Operator learning provides a natural viewpoint for this question. In PDE
settings, neural operators learn maps from problem specifications, such as
coefficients, forcing terms, or boundary data, to solution functions across an
entire family of equations \citep{kovachki2023neural}. Continuous symmetry has
a closely related differential characterization. If a function is invariant
under a continuous transformation group, then its infinitesimal generators
define directions along which the function does not change:
$\mathcal L_v f=0$. For a linear action generated by $A$, this reduces to
$\nabla f(x)^\top A x=0$.

We therefore view symmetry discovery as an inverse operator problem. Rather
than starting from a known differential constraint and solving for a compatible
function, we observe a function through finite samples and infer the
infinitesimal transformations that leave it invariant. The resulting inverse
symmetry operator maps functions to their symmetry structure, while preserving
the geometric constraints that characterize that structure.

This viewpoint also imposes structure beyond ordinary function-to-function
operator learning. For global Lie-group symmetries, the recovered Lie algebra
must transform equivariantly under changes of coordinates. More generally, the
relevant symmetry directions need not be fixed across the domain. We therefore
extend the formulation to local symmetry geometry, where the target is a
smoothly varying rank-$k$ field of tangent subspaces. Under a Lie-bracket
closure condition, Frobenius' theorem guarantees that these local directions
integrate into $k$-dimensional submanifolds along which the observables remain
unchanged.

Our contributions are:
\begin{enumerate}
    \item We formulate continuous symmetry discovery as an inverse operator
    problem, mapping a function to its symmetry Lie algebra, and derive a
    representation-theoretic population construction together with its
    equivariance under the ambient group action.
    \item We develop a structured amortization principle for finite-sample
    symmetry discovery: learn only the statistically difficult part of the
    inference pipeline while retaining known analytic symmetry geometry.
    \item We extend the framework from global Lie symmetries to local symmetry
    distributions, characterized as common Jacobian nullspaces and connected to
    involutive distributions through Frobenius' theorem.
\end{enumerate}

%% file: proposed-method.tex
\section{Problem Statement}
Let \(G\) be a Lie group acting smoothly on an input space \(X\), with Lie algebra \(\mathfrak g\). This action induces the usual pullback action on functions \(f:X\to\mathbb R\), given by \((g\cdot f)(x)=f(g^{-1}\cdot x)\).
For a particular function \(f\), we are interested in the subgroup of transformations that leave it unchanged,
\[
H_f=\{g\in G:\ g\cdot f=f\}.
\]In other words, \(f\) is invariant under \(H_f\), i.e., \(f(h\cdot x)=f(x)\) for every \(h\in H_f\). Since our focus is on continuous symmetries, it is more convenient to work with the corresponding Lie algebra \(\mathfrak h_f=\operatorname{Lie}(H_f)\subseteq\mathfrak g\).
The problem we study is: given only finitely many samples of an unknown function,
\[
\mathcal D_f=\{(x_i,f(x_i))\}_{i=1}^N,
\]recover its continuous symmetry structure, namely the subalgebra \(\mathfrak h_f\) and its dimension.

\section{Global Symmetry Discovery as an Inverse Operator}
\label{sec:global_symmetry}

Operator learning is commonly used for families of differential equations,
where a map from the specification of a problem to its solution function is
learned across many related instances. We take a related viewpoint for symmetry
discovery. A continuous symmetry has an infinitesimal differential
characterization: if $v_\xi$ is the vector field induced by
$\xi\in\mathfrak g$, then
\[
\xi\in\mathfrak h_f
\qquad\Longleftrightarrow\qquad
\mathcal L_{v_\xi}f=0.
\]
For a linear action on $X\subseteq\mathbb R^n$, with
$v_A(x)=Ax$, this becomes
\[
\nabla f(x)^\top Ax=0
\qquad\text{for all }x\in X.
\]

Here the direction of inference is reversed: rather than starting from a known
differential constraint and solving for a compatible function, we observe the
function and seek the infinitesimal transformations that leave it invariant.
We therefore define the inverse symmetry operator
\[
\mathcal S:f\mapsto\mathfrak h_f.
\]
Our aim is to infer this map from finite observations across a family of
functions, rather than solve a separate symmetry-estimation problem for every
new function.

\subsection{Representation-theoretic construction}

We first consider the idealized setting in which the function $f$ is known through its coefficients in a basis adapted to the action of $G$. This separates the geometric problem of identifying the symmetry subgroup from the statistical problem of estimating the function from finitely many samples.

More precisely, let $G$ act on the domain, with Lie algebra $\mathfrak{g}$, and consider a decomposition of the corresponding function space into representation spaces $V_\lambda$. Analogous to expanding a function in an orthonormal basis, we may write $f$ in terms of basis functions grouped according to the representations of $G$. We denote by
\[
c_\lambda \in V_\lambda
\]
the collection of coefficients associated with the representation indexed by $\lambda$, and write schematically
\[
f \longleftrightarrow \{c_\lambda\}_\lambda.
\]

For example, for functions on the sphere $\mathbb{S}^2$ under the action of $\mathrm{SO}(3)$, the spherical harmonic expansion is
\[
f(x)
=
\sum_{\ell=0}^{\infty}
\sum_{m=-\ell}^{\ell}
c_{\ell m}Y_{\ell m}(x), \qquad f \longleftrightarrow \{c_{l}\}_{l},\: \text{ with } c_{l} = \{c_{lm}\}_m.
\]

In general, the infinitesimal action of $\xi \in \mathfrak{g}$ on a
coefficient vector $c_\lambda$ is therefore
$d\rho_\lambda(\xi)c_\lambda$, where $d\rho_\lambda$ denotes the derived
representation. This motivates the linear map
\[
J_\lambda(c_\lambda)
:
\mathfrak{g}
\rightarrow
V_\lambda,
\qquad
J_\lambda(c_\lambda)\xi
=
d\rho_\lambda(\xi)c_\lambda.
\]

An infinitesimal generator $\xi$ is a symmetry of $f$ precisely when it annihilates every representation component. This motivates the positive-semidefinite operator
\[
K_f
=
\sum_\lambda
J_\lambda(c_\lambda)^*
J_\lambda(c_\lambda)
:
\mathfrak{g}
\rightarrow
\mathfrak{g}.
\]
Rather than searching directly over Lie subalgebras, $K_f$ converts symmetry discovery into a nullspace problem on the finite-dimensional Lie algebra $\mathfrak{g}$.


For a representation vector $c$, the classical isotropy result
\cite{tanimura2000reduction} gives
\[
\ker\!\left(J(c)^*J(c)\right)=\mathfrak h_c.
\]
Applying this blockwise yields the following function-level result.

Note: We assume that the representation decomposition is $G$-equivariant and
that the operator sum defining $K_f$ is finite or converges in the relevant
operator topology.
\begin{proposition}
Let
$f\leftrightarrow\{c_\lambda\}_{\lambda\in\Lambda}$ be a
$G$-equivariant representation decomposition and define
\[
K_f=
\sum_{\lambda\in\Lambda}
J_\lambda(c_\lambda)^*J_\lambda(c_\lambda),
\qquad
J_\lambda(c_\lambda)\xi=d\rho_\lambda(\xi)c_\lambda.
\]
Then
\[
\ker K_f
=
\bigcap_{\lambda\in\Lambda}\ker J_\lambda(c_\lambda)
=
\mathfrak h_f.
\]
\end{proposition}

The above result conveys that, if \(f\) decomposes into coefficient blocks \(\{c_\lambda\}\), then a Lie algebra element is a symmetry of \(f\) exactly when it stabilizes every block simultaneously. Summing the corresponding Gram operators therefore turns the intersection of these blockwise isotropy algebras into the nullspace of a single operator \(K_f\). The remaining challenge is to estimate from finite observations the quantities
needed to evaluate this map reliably.

\subsection{Equivariance of the inverse symmetry operator}
\label{sec:s_equivariance}

The inverse symmetry operator should transform consistently with the ambient
group action. For the pullback action
\[
(g\cdot f)(x)=f(g^{-1}x),
\]
the symmetry subgroup is conjugated by $g$, and hence the symmetry algebra
transforms by the adjoint action.

\begin{proposition}[Equivariance of the inverse symmetry operator]
For every $g\in G$,
\[
\mathcal S(g\cdot f)
=
\operatorname{Ad}_g\,\mathcal S(f).
\]
\end{proposition}

Thus a change of coordinates transports the recovered symmetry algebra rather
than leaving it invariant.

\subsection{Finite-sample realization}
\label{sec:finite_sample_global}

The preceding construction assumes access to the relevant representation of
$f$. In practice, we observe only
\[
\mathcal D_f=\{(x_i,f(x_i))\}_{i=1}^N,
\]
so the main difficulty is estimating from finite data the quantities required
by the analytic symmetry construction.

We therefore separate the finite-sample inference problem from the geometric
symmetry map. Some quantities can be estimated reliably by standard numerical
methods, while others may be noisy, poorly conditioned, or require
finite-sample calibration. We amortize only those parts that benefit from
information shared across functions, while retaining analytic structure
wherever it is available.

Schematically,
\[
\mathcal D_f
\longrightarrow
\text{finite-sample estimate}
\longrightarrow
\text{analytic symmetry construction}
\longrightarrow
\widehat{\mathfrak h}_f.
\]
The precise division between learned and analytic components depends on the
problem.

\subsection{$\mathrm{SO}(5)$: amortized harmonic inference}
\label{sec:so5_method}

For $G=\mathrm{SO}(5)$, we represent the angular dependence of $f$ using
hyperspherical-harmonic blocks
\[
f\longleftrightarrow\{c_\ell\}_{\ell\in\Lambda},
\]
where each $c_\ell$ transforms under the degree-$\ell$ representation of
$\mathrm{SO}(5)$.
The corresponding stabilizer operator is
\[
K_f
=
\sum_{\ell\in\Lambda}
J_\ell(c_\ell)^*J_\ell(c_\ell).
\]

We first consider a structured baseline in which the harmonic coefficients are
estimated independently for each function and only the finite-sample spectral
calibration is learned across functions.

For the harder setting, we additionally amortize the coefficient-estimation
step itself:
\[
\mathcal D_f
\xrightarrow{\mathcal E_\theta}
\{\widehat c_\ell\}_{\ell\in\Lambda}
\longrightarrow
\widehat K_f
\longrightarrow
\widehat{\mathfrak h}_f.
\]
The estimator $\mathcal E_\theta$ is equivariant at the level of each harmonic
block, while the stabilizer construction remains analytic.
Architecture and training details are given in
Appendix~\ref{app:so5_details}.

\subsection{$\mathrm{SL}(2)$: an analytic inverse}
\label{sec:sl2_method}

The framework does not require learning when the finite-sample inverse is
already well conditioned. For the $\mathrm{SL}(2)$ benchmark, $f$ is represented
in a finite polynomial basis whose coefficients are recovered by least squares.
The symmetry algebra then follows from the corresponding infinitesimal
stabilizer:
\[
\mathcal D_f
\longrightarrow
\widehat c^{\,\mathrm{LS}}
\longrightarrow
\widehat K_f
\longrightarrow
\widehat{\mathfrak h}_f.
\]
This provides a control showing that amortization is introduced only when it
improves finite-sample inference.

\section{Local Symmetry Discovery}
\label{sec:local_symmetry}

Global Lie-group symmetry is described by a fixed Lie subalgebra
$\mathfrak h_f\subseteq\mathfrak g$. More generally, the directions along
which a collection of observables remains locally unchanged may vary with
position. The resulting symmetry object is a field of tangent subspaces
$x\mapsto\mathcal D_F(x)\subseteq T_xM$. This leads to a second inverse
problem: recover this local symmetry geometry from finite observations.

\subsection{Local symmetry distributions}

Let $F=(f_1,\ldots,f_m):M\to\mathbb R^m$. At a regular point $x\in M$, define
\[
\mathcal D_F(x)
=
\bigcap_{r=1}^m \ker df_r(x)
=
\ker dF_x
=
\ker J_F(x),
\]
where $J_F(x)$ denotes the Jacobian of $F$ in local coordinates.

This is the local analogue of the global infinitesimal symmetry condition.
For a global linear generator $A$, invariance requires
$\nabla f(x)^\top Ax=0$ for all $x$. In the local setting, the symmetry
direction itself may vary with position and satisfies
$J_F(x)v(x)=0$.

\subsection{Identifiability and involutivity}

A single regular scalar observable identifies only the codimension-one space
$\ker df_x$, so recovering a lower-dimensional symmetry field generally
requires multiple observables. If $\operatorname{rank}dF_x=n-k$, then
$\dim\mathcal D_F(x)=k$.

The common-kernel construction also has a useful geometric consequence.
If vector fields $X$ and $Y$ satisfy $Xf_r=Yf_r=0$ for every component $f_r$,
then
\[
[X,Y]f_r
=
X(Yf_r)-Y(Xf_r)
=
0.
\]
Hence, under the constant-rank assumption, $\mathcal D_F$ is closed under Lie
brackets and is therefore involutive. By Frobenius' theorem, it locally forms
the tangent spaces of $k$-dimensional integral submanifolds along which all
components of $F$ remain constant.

\subsection{Finite-sample realization}

At population level, local symmetry recovery follows the explicit map
\[
F \longmapsto J_F(x) \longmapsto \mathcal D_F(x).
\]
With finite observations
$\mathcal O_F=\{(x_i,F(x_i))\}_{i=1}^N$, the main difficulty is estimating
$J_F(x)$ reliably.

A fully analytic estimator uses weighted local polynomial regression to obtain
$\widehat J_F(x)$ and then recovers the symmetry subspace from its nullspace:
\[
(\mathcal O_F,x)
\longrightarrow
\widehat J_F(x)
\longrightarrow
\widehat{\mathcal D}_F(x).
\]
We additionally consider a hybrid amortized estimator that learns only the
sample weights used in this local regression,
\[
(\mathcal O_F,x)
\xrightarrow{\mathcal W_\theta}
\{w_i\}_{i=1}^N
\longrightarrow
\widehat J_F(x)
\longrightarrow
\widehat{\mathcal D}_F(x).
\]
Thus the learned component improves the finite-sample differentiation step
without directly predicting the tangent subspace; the Jacobian-nullspace map
remains analytic.

The global and local settings therefore share the same inverse viewpoint but
have different outputs: $f\mapsto\mathfrak h_f$ in the global case and
$F\mapsto\mathcal D_F(\cdot)$ in the local case. The latter is a field of
tangent subspaces rather than a single fixed Lie algebra.

%% file: Experiments-v1.tex
\section{Experiments and Results}
\label{sec:experiments}


We evaluate the framework on global $\SO(5)$ symmetry discovery,
non-compact $\SL(2,\mathbb R)$ symmetries, spatially varying local symmetries,
and two physically motivated synthetic energy families. We report
symmetry-dimension and subspace-recovery metrics, using disjoint
train/validation/development/test functions and validation-only model
selection. Additional benchmarks, ablations, and implementation details are
given in Appendix~D.

For global symmetries, we report symmetry-dimension accuracy, hard-projector
error, oracle-dimension projector error, and principal-angle error. The
oracle-dimension metric evaluates the recovered subspace while supplying the
true symmetry dimension, thereby separating geometric recovery from dimension
selection. For local symmetries, we report projector error and largest principal
angle. All train, validation, development, and test splits are disjoint at the
level of generated functions. Hyperparameters and spectral thresholds are
selected using validation data only. For the principal final evaluations,
predictions are frozen before the held-out oracle quantities are used for
scoring. Additional $\SO(3)$ and $\SO(4)$ benchmarks, semi-supervised
experiments, ablations, and implementation details are deferred to the
appendix.

\subsection{Global symmetry discovery on $\SO(5)$}
\label{sec:exp_so5}

Our main global benchmark consists of functions on $S^4$ with stabilizer
dimensions $d_f\in\{0,1,2,3,4,10\}$. The nine classes include trivial,
one-dimensional, product, embedded $\SO(3)$, $\SO(3)\times\SO(2)$, and full
$\SO(5)$ stabilizers, together with three coupled one-generator classes with
relative rotation rates $1{:}1$, $1{:}2$, and $2{:}3$. All nontrivial
stabilizers are randomly conjugated, and the identifying angular information
extends to higher-order hyperspherical-harmonic blocks.

We compare four finite-sample estimators: independent analytic harmonic fitting,
a Amortized Harmonic Calibration-only model (AHC), the proposed Amortized Equivariant Harmonic Estimator (AEHE), and a Beyond-Affine/Killing-field baseline (BA). AEHE is trained on
$40{,}005$ functions and evaluated on a disjoint $2{,}016$-function final set.
Population coefficients recover the correct stabilizer for every benchmark
class.

Figure~\ref{fig:main_results}(a,b) shows that AEHE is substantially more
reliable for symmetry-dimension inference in the sparse regime and continues
to improve with sample count. Beyond-Affine becomes geometrically competitive
once its polynomial surrogate is sufficiently determined, but its
hard-projector error eventually plateaus. AEHE also improves over amortizing harmonic calibration (AHC) alone, indicating that the main finite-sample bottleneck in this benchmark is coefficient estimation
rather than the analytic stabilizer map.


The coupled one-dimensional classes provide an additional diagnostic because
the relative rotation rates must be recovered, not merely the symmetry
dimension. At the largest sample count, all three coupled classes recover the
correct dimension and achieve low generator-angle error. Detailed per-class
angles, recovered rate ratios, and the LieGAN comparison are reported in
Appendix~\ref{app:liegan_comparison}.

\begin{figure}[t]
    \centering
    \begin{subfigure}[t]{0.45\textwidth}
        \centering
        \includegraphics[width=\linewidth]
        {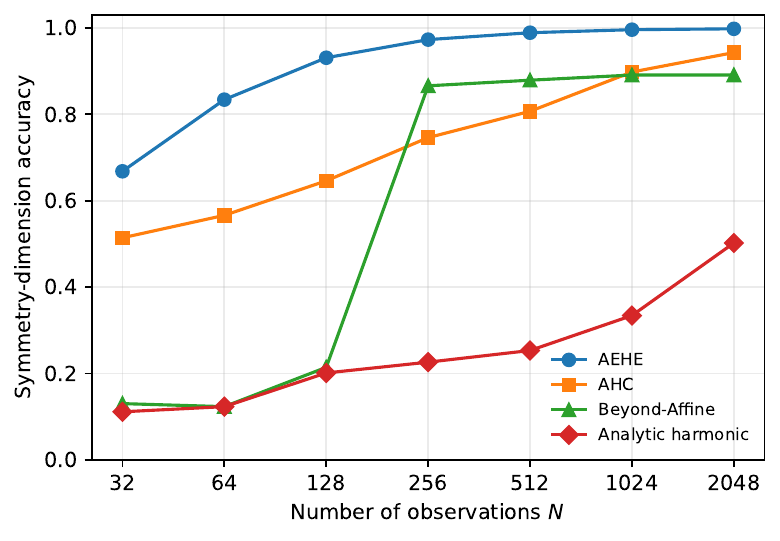}
        \caption{$\SO(5)$ symmetry-dimension accuracy.}
        \label{fig:so5_dim_curve}
    \end{subfigure}
    \hfill
    \begin{subfigure}[t]{0.45\textwidth}
        \centering
        \includegraphics[width=\linewidth]
        {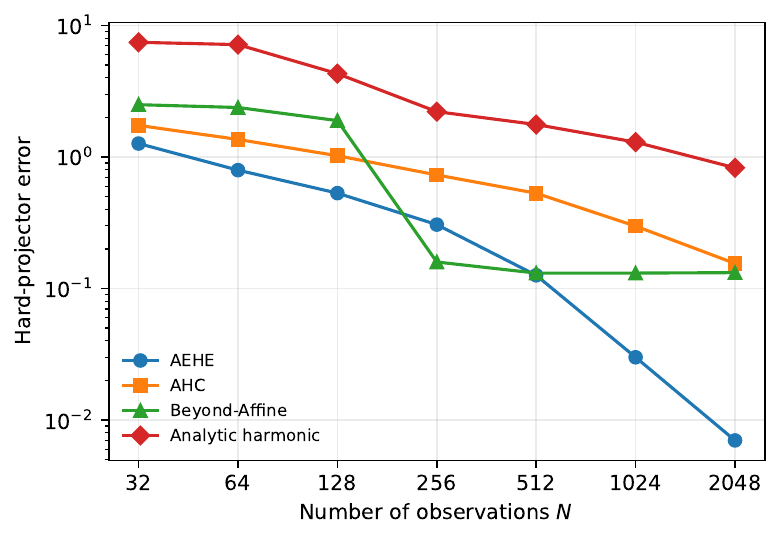}
        \caption{$\SO(5)$ hard-projector error.}
        \label{fig:so5_proj_curve}
    \end{subfigure}

    \vspace{0.5em}

    \begin{subfigure}[t]{0.45\textwidth}
        \centering
        \includegraphics[width=\linewidth]
        {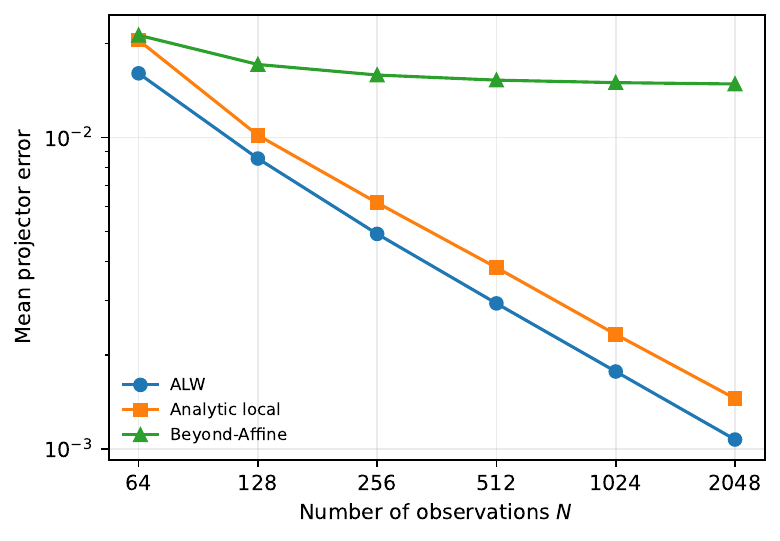}
        \caption{Local-symmetry projector error.}
        \label{fig:local_proj_curve}
    \end{subfigure}
    \hfill
    \begin{subfigure}[t]{0.45\textwidth}
        \centering
        \includegraphics[width=\linewidth]
        {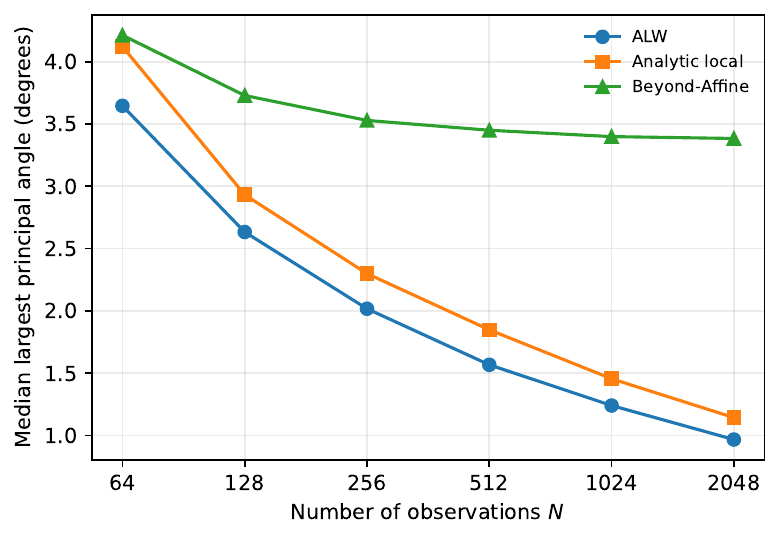}
        \caption{Local-symmetry principal-angle error.}
        \label{fig:local_angle_curve}
    \end{subfigure}

    \caption{
    Finite-sample symmetry recovery.
    Top: global $\SO(5)$ symmetry discovery.
    Bottom: recovery of spatially varying local symmetry distributions.
    Amortization is most effective when applied to the statistically difficult
    finite-sample inference step while retaining the analytic symmetry
    construction.
    }
    \label{fig:main_results}
\end{figure}


\paragraph{Beyond compact groups.}
To verify that the stabilizer construction is not specific to compact
orthogonal groups, we also evaluate it on an $\SL(2,\mathbb R)$ benchmark with
elliptic, hyperbolic, parabolic, and trivial stabilizers. In this setting,
finite polynomial coefficients can be recovered reliably by ordinary least
squares, after which the symmetry algebra follows analytically from the
infinitesimal representation. The experiment therefore serves as a control:
the geometric construction extends beyond $\SO(n)$, and amortized learning is
unnecessary when the finite-sample inverse is already well conditioned.
Full results are given in Appendix~\ref{app:sl2_experiment}.

\subsection{Local symmetry distributions}
\label{sec:exp_local}

We next evaluate the extension from a single global Lie algebra to a
position-dependent local symmetry distribution
\[
\mathcal D_F(x)=\ker dF_x.
\]
The hard benchmark uses
$
F:\mathbb R^6\rightarrow\mathbb R^4,
\text{ and }
\dim\mathcal D_F(x)=2,
$
with function-specific nonlinear diffeomorphisms producing strongly curved
rank-two tangent distributions. The dataset contains
$5{,}000/512/512$ train/validation/test functions.

We compare three main approaches: (i) \emph{Analytic local}, which uses fixed
weighted local quadratic regression followed by an analytic
Jacobian-nullspace computation; (ii) Beyond-Affine, using two global
polynomial vector fields; and (iii) \emph{Amortized Local Weighting} (ALW),
which learns query-dependent sample weights while retaining the local
polynomial solve, Jacobian estimation, and nullspace computation analytically.

As shown in Figure~\ref{fig:local_proj_curve}-~\ref{fig:local_angle_curve} and Table~\ref{tab:so5_full_final}, the hybrid estimator consistently
improves over fixed analytic differentiation across all sample sizes, both in
projector error and principal-angle recovery. Beyond-Affine is competitive in
the sparsest regime but plateaus as the target distributions become more
nonlinear, whereas the analytic and hybrid estimators continue to improve with
additional observations.

We additionally evaluate direct neural projector, which  performs substantially worse and is reported
in the appendix. The comparison supports the same design principle observed in
the global setting: learning is most effective when applied to the difficult
finite-sample estimation step, while the known geometric map from the estimated
Jacobian to the symmetry subspace is retained analytically. 


\subsection{Physically motivated synthetic applications}
\label{sec:exp_physical}

We finally evaluate the same framework on two simulated physical
proof-of-concept families with exactly known continuous stabilizers. These are
not intended as full micromagnetic or quantum-material simulations; rather,
they provide physically interpretable energy landscapes for which the symmetry
ground truth remains analytically controlled.

\paragraph{$\SO(4)$ magnetic orientation energy.}
A rigid magnetic orientation is represented by a unit quaternion
$q\in S^3$, with $q\sim -q$. We construct orientation-dependent energies
$E:S^3/\{\pm1\}\to\mathbb R$ whose left--right quaternion symmetries realize
stabilizer dimensions $0,1,2,3,$ and $6$. Subgroup orientations are randomized
across functions.

\paragraph{$\SO(5)$ superspin effective free energy.}
We similarly consider a normalized five-component superspin
$n=(N_x,N_y,N_z,V_x,V_y)\in S^4$ and synthetic Landau-type angular free
energies with stabilizer dimensions $0,1,2,3,4,$ and $10$. The
dimension-four class realizes $\SO(3)_N\times\SO(2)_V$, while the fully
isotropic limit has exact $\SO(5)$ symmetry. Oracle infinitesimal checks verify
all claimed stabilizers.

\begin{table}[t]
\centering
\caption{
Physically motivated synthetic applications.
Entries report dimension accuracy / hard-projector error.
We show a sparse noisy setting and a well-sampled clean setting;
full sample-size and noise sweeps are reported in the appendix.
}
\label{tab:physical_main}
\small
\setlength{\tabcolsep}{4pt}

\begin{tabular*}{\linewidth}{@{\extracolsep{\fill}}lcc|cc@{}}
\toprule
& \multicolumn{2}{c|}{$N=32,\ \sigma=.05$}
& \multicolumn{2}{c}{$N=1024$, clean} \\
Method
& $\SO(4)$ & $\SO(5)$
& $\SO(4)$ & $\SO(5)$ \\
\midrule
AEHE
& .869/.598 & .586/1.410
& .961/.179 & .892/.135 \\

AHC
& .948/.417 & .546/1.452
& 1.000/.00241 & .998/.0126 \\

Analytic harmonic
& .200/2.400 & .167/3.333
& 1.000/$4.9{\times}10^{-7}$
& 1.000/$4.0{\times}10^{-17}$ \\

Beyond-Affine
& .288/1.886 & .318/2.563
& 1.000/$2.09{\times}10^{-5}$
& 1.000/$3.71{\times}10^{-6}$ \\
\bottomrule
\end{tabular*}
\end{table}

\begin{table}[t]
\centering
\caption{
Oracle-dimension geometry at $N=1024$ and $\sigma=.05$.
Supplying the true symmetry dimension isolates geometric subspace recovery
from dimension selection.
}
\label{tab:physical_oracle}
\small
\setlength{\tabcolsep}{4pt}

\begin{tabular*}{\linewidth}{
@{\extracolsep{\fill}}
l
cc
cc
@{}
}
\toprule
& \multicolumn{2}{c}{$\SO(4)$ magnetic}
& \multicolumn{2}{c}{$\SO(5)$ superspin} \\
Method
& Proj. & Angle
& Proj. & Angle \\
\midrule
AEHE
& .1499 & $12.53^\circ$
& .0109 & $4.03^\circ$ \\

AHC
& .00244 & $1.98^\circ$
& .0112 & $4.08^\circ$ \\

Beyond-Affine
& $5.36{\times}10^{-5}$ & $0.19^\circ$
& $3.66{\times}10^{-4}$ & $0.51^\circ$ \\
\bottomrule
\end{tabular*}
\end{table}

Table~\ref{tab:physical_main} shows a clear regime dependence. With clean,
well-sampled measurements, direct polynomial and harmonic estimators can exploit
the low-degree structure of these synthetic energy families and recover the
symmetry geometry almost exactly. Under observation noise, however, their
automatic dimension-selection rules become brittle, particularly for the
full-symmetry classes. Amortized harmonic calibration (AHC) is substantially more
robust in this regime and provides reliable end-to-end symmetry recovery.

The oracle-dimension results in Table~\ref{tab:physical_oracle} clarify the
failure mode. Beyond-Affine continues to estimate the nontrivial symmetry
subspace very accurately under noise; most of its end-to-end error arises from
selecting the wrong symmetry dimension. Thus the benefit of learned calibration
is primarily robust finite-sample dimension inference rather than uniformly
better subspace geometry.
\vspace{-0.6cm}
\paragraph{Controlled symmetry enhancement.}
For the superspin benchmark we additionally consider
$F_\epsilon=F_{\SO(5)}+\epsilon A$, where $A$ breaks the full symmetry to
$\SO(3)\times\SO(2)$. The exact stabilizer dimension is four for every
$\epsilon\neq0$ and jumps to ten only at $\epsilon=0$. Nevertheless, the
symmetry-breaking spectrum varies continuously: the first nonzero physics-Gram
eigenvalue scales as $\epsilon^2$, while the learned soft full-symmetry score
approaches one as $\epsilon\to0$. We therefore interpret the learned score as
an indicator of approximate symmetry enhancement, not as a continuously varying
exact symmetry dimension.

%% file: supplementary.tex
\subsection{Architecture and training details for \texorpdfstring{$\mathrm{SO}(5)$}{SO(5)}}
\label{app:so5_details}

\paragraph{Harmonic representation.}
We use the real hyperspherical-harmonic degrees
\[
\Lambda=\{2,4,5,6\}.
\]
Degrees \(1\) and \(3\) are excluded because they vanish identically for the
frozen benchmark family.  For \(S^4\), the corresponding representation
dimensions are
\[
\dim\mathcal H_2=14,~
\dim\mathcal H_4=55,~
\dim\mathcal H_5=91,~
\dim\mathcal H_6=140.
\]
The Lie algebra \(\mathfrak{so}(5)\) is represented in a fixed
Frobenius-orthonormal bivector basis of dimension \(10\).  For each harmonic
coefficient vector \(c_\ell\), we form
\[
J_\ell(c_\ell)
=
\begin{bmatrix}
L^{(\ell)}_1c_\ell &
\cdots &
L^{(\ell)}_{10}c_\ell
\end{bmatrix},
\]
where \(L^{(\ell)}_j\) is the infinitesimal degree-\(\ell\) representation of
the \(j\)-th Lie-algebra basis element.  The degree-wise stabilizer metric is
\[
A^{(\ell)}
=
J_\ell(c_\ell)^\top J_\ell(c_\ell),
\qquad
A^{(\ell)}_{jk}
=
\left\langle
L^{(\ell)}_j c_\ell,\,
L^{(\ell)}_k c_\ell
\right\rangle.
\]
Thus,
\[
a^\top A^{(\ell)}a
=
\left\|
d\rho_\ell\!\left(\sum_{j=1}^{10}a_jE_j\right)c_\ell
\right\|_2^2.
\]
The combined metric is positive semidefinite, and its nullspace estimates the
continuous symmetry algebra.

\paragraph{Radial component removal.}
Before estimating angular coefficients, we remove the radial component
independently for each sampled function.  Let
\[
q_i=\|x_i\|_2^2.
\]
We regress the sampled values onto the degree-six shifted-Legendre basis in
\(q\),
\[
\widehat g_f(q)=\sum_{k=0}^{6}\beta_{f,k}\phi_k(q),
\]
using a differentiable least-squares solve, and define
\[
r_i=f(x_i)-\widehat g_f(q_i).
\]
The seven fitted radial coefficients are retained as invariant context
features.  This detrending is analytic and is unchanged between the structured
calibration and amortized-coefficient models.

\subsubsection{Structured harmonic calibration}

The structured model estimates each harmonic block independently from the
samples and learns only how the finite-sample blocks should be combined.
For each \(\ell\in\Lambda\), the finite-sample coefficient estimate is
\[
\widetilde c_\ell
=
\frac{1}{N}\sum_{i=1}^{N}
r_iY_\ell(\widehat x_i),
\qquad
\widehat x_i=\frac{x_i}{\|x_i\|_2}.
\]
Applying the fixed infinitesimal representation matrices gives
\[
\widetilde A^{(\ell)}
=
J_\ell(\widetilde c_\ell)^\top
J_\ell(\widetilde c_\ell).
\]

\paragraph{Invariant gate features.}
The learned calibrator receives only rotation-invariant scalar summaries.
For each degree, these include:
\begin{itemize}
    \item logarithmic coefficient power;
    \item estimated signal and noise traces;
    \item trace-based signal-to-noise ratio;
    \item signal and noise Frobenius norms;
    \item the ten ordered, trace-normalized signal eigenvalues;
    \item the ten ordered, noise-normalized eigenvalues.
\end{itemize}
The signal and noise quantities are obtained from split estimates of the
harmonic block.  Across the four degrees, these features contribute
\(4\times 26=104\) scalars.  Six normalized pairwise block contractions and
\(\log N\) are appended, giving \(111\) invariant inputs in total.

A width-\(64\) GELU multilayer perceptron predicts four nonnegative softmax
weights
\[
w_\ell(f)\geq 0,
\qquad
\sum_{\ell\in\Lambda}w_\ell(f)=1,
\]
and six symmetry-dimension logits corresponding to
\[
\mathcal R=\{0,1,2,3,4,10\}.
\]
The calibrated stabilizer metric is
\[
K_\theta(f)
=
\sum_{\ell\in\Lambda}
w_\ell(f)
\frac{\widetilde A^{(\ell)}}
     {\operatorname{tr}\widetilde A^{(\ell)}+\varepsilon}.
\]
The model contains \(11{,}978\) trainable parameters.

Because the learned quantities \(w_\ell\) are invariant scalars,
\[
K_\theta(R\cdot f)
=
\operatorname{Ad}_R
K_\theta(f)
\operatorname{Ad}_R^\top.
\]
For a candidate dimension \(d\in\mathcal R\), let
\(P_d(K_\theta)\) denote the projector onto the eigenspace spanned by the
\(d\) smallest eigenvectors of \(K_\theta\).  The hard prediction uses
\[
\widehat d=\arg\max_{d\in\mathcal R}p_\theta(d\mid f),
\qquad
\widehat P=P_{\widehat d}(K_\theta).
\]
We also form the differentiable soft projector
\[
\widehat P_{\mathrm{soft}}
=
\sum_{d\in\mathcal R}
p_\theta(d\mid f)\,P_d(K_\theta).
\]

\paragraph{Calibration loss and optimization.}
The structured model is trained with
\[
\mathcal L_{\mathrm{cal}}
=
\mathcal L_{\mathrm{CE}}
+
\frac{1}{10}
\left\|
\widehat P_{\mathrm{soft}}-P^\star
\right\|_F^2
+
\frac{1}{10}
\left\|
P_{d^\star}(K_\theta)-P^\star
\right\|_F^2,
\]
where \(d^\star\) and \(P^\star\) are the true symmetry dimension and
projector.  The final term evaluates the learned metric geometry while
supplying the oracle dimension during training.  Division by
\(\dim\mathfrak{so}(5)=10\) keeps the projector losses on an order-one scale.

Training uses AdamW with learning rate \(10^{-3}\), weight decay \(10^{-4}\),
batch size \(256\), and at most \(150\) epochs.  Early stopping uses patience
\(25\).  We train three independent optimization seeds,
\[
1729,\qquad 271828,\qquad 314159.
\]
The training and validation sets contain respectively \(128\) and \(64\)
functions per subclass, balanced over the nine benchmark subclasses.  This
gives \(1{,}152\) training and \(576\) validation functions.  Each function is
evaluated using nested prefixes
\[
N\in\{32,64,128,256,512,1024,2048\}.
\]
Checkpoint selection uses validation symmetry-dimension accuracy, with mean
hard-projector error as the tie-breaker.  No final-test function is used for
model or hyperparameter selection.

\subsubsection{Amortized equivariant coefficient operator}

The amortized model replaces the independent coefficient fit with a shared,
function-level equivariant estimator.  It preserves the exact harmonic
carriers and analytic stabilizer construction; learning is restricted to
invariant scalar quadrature weights and symmetry-dimension prediction.

\paragraph{Per-observation invariant features.}
For observation \(i\), the base invariant feature vector is
\[
s_i=
\left[
r_i^{(x)},\,
q_i,\,
q_i^2,\,
q_i^3,\,
\overline f_i,\,
\overline f_i^2,\,
\overline r_i,\,
\overline r_i^2,\,
q_i\overline r_i,\,
\log \sigma_r
\right],
\]
where
\[
r_i^{(x)}=\|x_i\|_2,\qquad
q_i=\|x_i\|_2^2,
\]
\(\overline f_i\) is the centered and standardized function value,
\(\overline r_i\) is the normalized radial-LS residual, and
\(\sigma_r\) is the per-function residual scale.

The frozen V2 architecture appends seven invariant all-pairs statistics:
\begin{enumerate}
    \item the mean first-order contraction \(x_i^\top x_j\);
    \item the mean squared contraction \((x_i^\top x_j)^2\);
    \item the mean cubic contraction \((x_i^\top x_j)^3\);
    \item first-order contractions weighted by standardized values;
    \item second-order contractions weighted by standardized values;
    \item first-order contractions weighted by radial residuals;
    \item second-order contractions weighted by radial residuals.
\end{enumerate}
These quantities are evaluated through pooled first-, second-, and third-order
Cartesian moments rather than by materializing an \(N\times N\) matrix.
Consequently, the implementation remains permutation invariant and avoids
quadratic memory in \(N\).  The resulting local feature dimension is \(17\).

\paragraph{Encoder and invariant context.}
A shared three-layer SiLU MLP with hidden width \(96\) maps each
\(s_i\in\mathbb R^{17}\) to a \(96\)-dimensional local embedding \(h_i\).
We compute both mean and maximum pooling,
\[
\overline h=\frac1N\sum_i h_i,
\qquad
h_{\max}=\max_i h_i.
\]
The pooled features are concatenated with the seven radial-LS coefficients
and the logarithm of the function-value scale.  A second width-\(96\) SiLU
MLP produces a \(96\)-dimensional invariant global context \(z_f\).

\paragraph{Equivariant coefficient prediction.}
For every \(\ell\in\{2,4,5,6\}\), a degree-specific width-\(96\) MLP receives
\[
[s_i,h_i,z_f]
\]
and predicts a positive scalar gate
\[
\alpha_{\ell i}
=
\operatorname{softplus}
\left(
g_{\ell,\theta}(s_i,h_i,z_f)
\right)
+10^{-5}.
\]
The predicted harmonic block is
\[
\widehat c_\ell
=
\frac1N\sum_{i=1}^{N}
r_i\alpha_{\ell i}Y_\ell(\widehat x_i).
\]
All learned quantities multiplying \(Y_\ell\) are rotation-invariant scalars.
Therefore,
\[
\widehat c_\ell(R\mathcal D_f)
=
D_\ell(R)\widehat c_\ell(\mathcal D_f)
\]
exactly, up to floating-point error.  The predicted stabilizer is constructed
analytically:
\[
\widehat K_f
=
\sum_{\ell\in\Lambda}
J_\ell(\widehat c_\ell)^\top
J_\ell(\widehat c_\ell).
\]
A separate invariant MLP maps \(z_f\) to six logits over the allowed symmetry
dimensions.  The complete V2 operator contains \(195{,}754\) trainable
parameters.

\paragraph{Coefficient targets.}
Training targets are population harmonic coefficients obtained from exact
angular projection of the frozen anisotropic polynomial family together with
analytic unit-ball radial integrals.  The targets therefore do not contain
finite-sample projection noise.  They are used only during training and
validation and are never supplied at inference time.

For each degree, let
\[
e_\ell
=
\|\widehat c_\ell-c_\ell^\star\|_2^2,
\qquad
p_\ell
=
\|c_\ell^\star\|_2^2.
\]
The active-target coefficient loss is normalized by target power:
\[
\mathcal L_{\ell,\mathrm{active}}
=
\frac{e_\ell}
{p_\ell+\epsilon\,\overline p_{\ell,\mathrm{active}}},
\qquad
\epsilon=0.05.
\]
For classes in which a block vanishes exactly, leakage is penalized relative
to the mean active-block scale:
\[
\mathcal L_{\ell,\mathrm{zero}}
=
\frac{\|\widehat c_\ell\|_2^2}
{\overline p_{\ell,\mathrm{active}}}.
\]
The total coefficient loss is
\[
\mathcal L_{\mathrm{coeff}}
=
\sum_{\ell\in\Lambda}
\left(
\mathcal L_{\ell,\mathrm{active}}
+
0.1\,\mathcal L_{\ell,\mathrm{zero}}
\right).
\]

\paragraph{Two-stage optimization.}
Training proceeds in two stages:
\begin{align}
\text{Stage A:}\qquad
&\mathcal L_A=\mathcal L_{\mathrm{coeff}},\\
\text{Stage B:}\qquad
&\mathcal L_B=
\mathcal L_{\mathrm{coeff}}
+
0.2\,\mathcal L_{\mathrm{CE}},
\end{align}
where \(\mathcal L_{\mathrm{CE}}\) is cross entropy over the six allowed
symmetry dimensions.  We deliberately do not add a direct projector loss.
The learned quantity of interest is the representation coefficient, while
the map from coefficients to the stabilizer metric remains analytic.
Avoiding a projector loss also removes the need to differentiate through
eigenvalue degeneracies.

The optimizer is AdamW with
\[
\text{learning rate}=3\times10^{-4},
\qquad
\text{weight decay}=10^{-5},
\qquad
\text{batch size}=16.
\]
Gradients are clipped to norm \(5\).  Stage A runs for at most \(30\) epochs
and Stage B for at most \(40\) epochs.  Validation-only early stopping starts
after epoch \(8\), with patience \(7\) and minimum improvement
\(10^{-4}\).  The frozen optimization seed is \(1729\).

A single model is trained across
\[
N\in\{32,64,128,256,512\}
\]
and evaluated without modification at
\[
N\in\{32,64,128,256,512,1024,2048\}.
\]
The scaled training and validation sets contain \(40{,}005\) and \(4{,}005\)
functions, respectively, balanced over the nine subclasses.  An independent
\(1{,}008\)-function development set is used for the frozen architecture
decision.  Train, validation, development, and final-test function seeds and
identifiers are disjoint.

\paragraph{Architecture freeze.}
An initial V1 model used only the ten local radial/value features.  Its
development audit recovered the degree-\(2\) block accurately but
underestimated degrees \(4\), \(5\), and \(6\).  The sole architecture
amendment was the addition of the seven invariant pairwise contractions
described above.  This V2 model was then frozen; no further architecture or
hyperparameter search was conducted.  In particular, the final-test data were
generated only after the V2 checkpoint and evaluation protocol had been
fixed.

\paragraph{Exactness tests.}
For both models, unit tests verify:
\begin{itemize}
    \item invariance to observation permutation;
    \item operation at variable sample count \(N\);
    \item covariance of every harmonic block,
    \[
    \widehat c_\ell(R\mathcal D_f)
    =
    D_\ell(R)\widehat c_\ell(\mathcal D_f);
    \]
    \item adjoint covariance of the stabilizer metric,
    \[
    \widehat K_{R\cdot f}
    =
    \operatorname{Ad}_R
    \widehat K_f
    \operatorname{Ad}_R^\top;
    \]
    \item covariance, symmetry, rank, and idempotence of the predicted
    projector;
    \item positive semidefiniteness of each
    \(J_\ell(\widehat c_\ell)^\top J_\ell(\widehat c_\ell)\);
    \item finite gradients through radial detrending, invariant pooling, and
    coefficient prediction.
\end{itemize}
For the trained V2 model, the measured relative equivariance errors for
degrees \(2,4,5,6\) were respectively
\[
3.38\times10^{-7},\quad
1.04\times10^{-6},\quad
2.17\times10^{-6},\quad
1.31\times10^{-6}.
\]
The corresponding errors for the stabilizer metric and projector were
\(3.84\times10^{-7}\) and \(5.10\times10^{-6}\).

\section{Additional results}
\subsection{Summary of Invariant Functions}

Here we provide complete details of the invariant functions used in the various experiments.

\subsubsection{$\mathrm{SO}(3)$ families}

\paragraph{Benchmark A.}
Let $v\in S^2$, $s=v^\top x$, and
$r_\perp^2=\|x\|^2-s^2$.  The initial axisymmetric family is
\begin{equation}
 f(x)=c_0+c_1s+c_2s^2+c_3r_\perp^2+c_4sr_\perp^2+c_5s^3+c_6r_\perp^4.
 \label{eq:benchmark_a_family}
\end{equation}
The coefficients are sampled from centered Gaussians of standard deviation
$0.5$, except that $|c_1|\sim\operatorname{Unif}(0.5,1.5)$ with random sign.
The nonzero linear term removes accidental enhancement beyond rotations about
$v$.  The symmetry projector in $\mathfrak{so}(3)\simeq\mathbb R^3$ is
$vv^\top$.

\paragraph{Higher-order axial family.}
The no-low-order-shortcut family uses $q=\|x\|^2$, $s=v^\top x$, and
\begin{align}
 g(q)&=c_0+c_1q+c_2q^2+c_3\sin(\omega q),\\
 h(s,q)&=a_4s^4+a_6s^6+a_{4q}qs^4,\\
 f(x)&=g(q)+e^{-\beta q}h(s,q).
 \label{eq:so3_higher_order_family}
\end{align}
Here $c_0,c_1\sim\operatorname{Unif}(-1,1)$,
$c_2\sim\operatorname{Unif}(-.5,.5)$,
$|c_3|\sim\operatorname{Unif}(.2,.8)$ with random sign,
$|a_4|\sim\operatorname{Unif}(.5,1.5)$ with random sign,
$a_6,a_{4q}\sim\operatorname{Unif}(-.8,.8)$,
$\beta\sim\operatorname{Unif}(.3,1.2)$, and
$\omega\sim\operatorname{Unif}(.75,2.5)$.  The intermediate benchmark keeps
the same draws and multiplies $(a_4,a_6,a_{4q})$ by three.  Because the
anisotropy begins at degree four,
\[
 \nabla f(0)=0,\qquad
 \nabla^2f(0)=2(c_1+c_3\omega)I_3.
\]

\paragraph{Mixed full/axial/trivial family.}
For $\ell\in\{2,4,6\}$ define the solid zonal harmonics
\begin{align}
 Z_2(v;x)&=\tfrac12(3s^2-q),\\
 Z_4(v;x)&=\tfrac18(35s^4-30qs^2+3q^2),\\
 Z_6(v;x)&=\tfrac1{16}(231s^6-315qs^4+105q^2s^2-5q^3),
\end{align}
where $s=v^\top x$.  The sampled function is
\begin{equation}
 f(x)=g(q)+\sum_{\ell\in\{2,4,6\}}
 a_\ell e^{-\beta_\ell q}(1+\gamma_\ell q)Z_\ell(v_\ell;x).
 \label{eq:so3_mixed_family}
\end{equation}
For the full-symmetry class all $a_\ell$ are inactive.  For the axial class,
$v_2=v_4=v_6$, giving $d_f=1$.  For the trivial class,
$v_2=v_4=v$ but $v_6=w$, where
$\angle(v,w)\sim\operatorname{Unif}(30^\circ,60^\circ)$; the incompatible
axes have zero continuous stabilizer.  In one half of axial and trivial
examples, $a_2=0$ exactly.  Before this masking,
$|a_2|\sim\operatorname{Unif}(.5,1.2)$ and
$|a_4|,|a_6|\sim\operatorname{Unif}(.8,1.4)$ with independent random signs;
$\beta_\ell\sim\operatorname{Unif}(.3,1.2)$ and
$\gamma_\ell\sim\operatorname{Unif}(-.5,.5)$.  Radial nuisance draws are
matched exactly across the three classes.  All axes are globally Haar rotated,
and observations are uniform in volume in the unit ball.

\subsubsection{$\mathrm{SO}(4)$ families}

\paragraph{Quadratic oracle family.}
Experiment 2A uses
\begin{equation}
 f(x)=x^\top R\operatorname{diag}(\lambda^{(d)})R^\top x,
 \qquad R\sim\operatorname{Haar}(\mathrm{SO}(4)),
\end{equation}
with
\begin{align*}
\lambda^{(0)}&=(-1.50,-.35,.80,1.90),&
\lambda^{(1)}&=(-1.10,-1.10,.45,1.75),\\
\lambda^{(2)}&=(-1.20,-1.20,1.25,1.25),&
\lambda^{(3)}&=(-.85,-.85,-.85,1.40),\\
\lambda^{(6)}&=(.70,.70,.70,.70).
\end{align*}
The multiplicities give stabilizer dimensions $d\in\{0,1,2,3,6\}$.

\paragraph{Higher-order family.}
Let $z=R^\top x$ (equivalently $z=xR$ in the stored row-vector convention),
$q=\|x\|^2$, and
\begin{equation}
 f_d(x)=c_0+c_1q+c_2q^2+c_3\sin(\omega q)+e^{-\beta q}h_d(z).
 \label{eq:so4_common_family}
\end{equation}
The canonical anisotropies are
\begin{align}
h_6&=0,\\
h_3&=au^2+bu^3+cuw+dw^2,
&&u=z_1^2+z_2^2+z_3^2,\quad w=z_4^2,\\
h_2&=au^2+bw^2+cuw+du^3+ew^3,
&&u=z_1^2+z_2^2,\quad w=z_3^2+z_4^2,\\
h_1&=au^2+bv^2+cw^2+duv+euw+hvw,
&&u=z_1^2+z_2^2,\quad v=z_3^2,\quad w=z_4^2,\\
h_0&=\sum_{i=1}^4(a_i z_i^4+b_i z_i^6)
     +\sum_{1\le i<j\le4}c_{ij}z_i^2z_j^2.
\end{align}
The ranges are
\begin{align*}
d=3:&\quad a\in[.8,1.6],~b\in[-1.2,-.4],~c\in[-.6,.6],~d\in[.5,1.3],\\
d=2:&\quad a\in[.8,1.6],~b\in[-1.6,-.8],~c\in[-.6,.6],~d\in[.3,.9],~e\in[-.9,-.3],\\
d=1:&\quad a\in[.9,1.7],~b\in[-1.7,-.9],~c\in[.4,1],~d\in[.3,.9],~e\in[-.9,-.3],~h\in[-.5,.5].
\end{align*}
For $d=0$, $(a_i)$ and $(b_i)$ are the bases
$(1.40,.75,-.85,-1.45)$ and $(.45,-.65,.85,-1.05)$ multiplied by independent
scales in $[.85,1.15]$, and $c_{ij}\sim\operatorname{Unif}(-.45,.45)$.
The shared radial ranges are
$c_0\in[-1,1]$, $c_1\in[-.6,.6]$, $c_2\in[-.4,.4]$,
$|c_3|\in[.2,.6]$ with random sign, $\omega\in[.8,2.4]$, and
$\beta\in[.2,.8]$.  Again
$\nabla f_d(0)=0$ and
$\nabla^2 f_d(0)=2(c_1+c_3\omega)I_4$.

\subsubsection{$\mathrm{SO}(5)$ higher-order family}

The common form is identical to Eq.~\eqref{eq:so4_common_family} with
$x,z\in\mathbb R^5$.  For a partition
$\mathcal G=(G_1,\ldots,G_k)$ of the coordinate indices, define
$u_g=\sum_{j\in G_g}z_j^2$ and
\begin{equation}
 h_{\mathcal G}(z)=\sum_{g=1}^k(a_gu_g^2+b_gu_g^3)
 +\sum_{g<h}c_{gh}u_gu_h.
 \label{eq:so5_group_anisotropy}
\end{equation}
The canonical partitions are
\begin{align*}
d=0:&~(1)(2)(3)(4)(5),&
d=1:&~(12)(3)(4)(5),\\
d=2:&~(12)(34)(5),&
d=3:&~(123)(4)(5),\\
d=4:&~(123)(45),&
d=10:&~(12345),~~h\equiv0.
\end{align*}
Here $(a_g)$ and $(b_g)$ use the ordered bases
$(1.55,.75,-.85,-1.45,2.05)$ and
$(.45,-.65,.85,-1.05,1.25)$, each multiplied by an independent scale in
$[.82,1.18]$, while $c_{gh}\in[-.45,.45]$.  The radial ranges are
$c_0\in[-1,1]$, $c_1\in[-.6,.6]$, $c_2\in[-.4,.4]$,
$|c_3|\in[.2,.6]$, $\omega\in[.8,2.4]$, and $\beta\in[.2,.8]$.

The three coupled one-dimensional classes use the base partition
$(12)(34)(5)$ and add
\begin{equation}
 h_{p:q}^{\mathrm{cpl}}(z)
 =\operatorname{Re}\!\left[
 \kappa(z_1+iz_2)^q(z_3-iz_4)^p
 \right]m_{p:q}(z),
 \qquad \kappa=Ae^{i\varphi},
 \label{eq:so5_coupled_term}
\end{equation}
where
\begin{align*}
m_{1:1}(z)&=z_1^2+z_2^2+\gamma(z_3^2+z_4^2)+\delta z_5^2,\\
m_{1:2}(z)&=z_5,\\
m_{2:3}(z)&=1.
\end{align*}
We draw $A\in[1.5,3]$, $\varphi\in[-\pi,\pi]$,
$\gamma\in[.6,1.4]$, and $\delta\in[-.8,.8]$.  The preserved generator is
$pE_{12}+qE_{34}$, so these terms distinguish the relative rates rather than
only the invariant two-plane.  Finally, every canonical function is conjugated
by an independent Haar draw $R\in\mathrm{SO}(5)$.

\subsubsection{$\mathrm{SL}(2,\mathbb R)$ polynomial families}

Write $x=(x_1,x_2)$ and use the canonical generators
$A_{\rm ell}=E-F$, $A_{\rm hyp}=H$, and $A_{\rm par}=E$.  Their polynomial
first integrals are respectively
\[
 I_{\rm ell}=x_1^2+x_2^2,\qquad
 I_{\rm hyp}=x_1x_2,\qquad
 I_{\rm par}=x_2.
\]
The canonical functions are
\begin{align}
f_{\rm ell}(x)&=c_0+a_4 I_{\rm ell}^{2}+a_6I_{\rm ell}^{3}+a_8I_{\rm ell}^{4},\\
f_{\rm hyp}(x)&=c_0+a_4 I_{\rm hyp}^{2}+a_6I_{\rm hyp}^{3}+a_8I_{\rm hyp}^{4},\\
f_{\rm par}(x)&=c_0+a_4x_2^4+a_5x_2^5+a_6x_2^6,\\
f_{0}(x)&=c_0+a_4I_{\rm ell}^2+a_5x_2^5+a_6I_{\rm hyp}^3.
\end{align}
The final line intersects three incompatible degree-wise stabilizers and has
no continuous symmetry.  The magnitudes satisfy
$|a_4|\in[.7,1.3]$, $|a_5|\in[.35,.85]$,
$|a_6|\in[.25,.70]$, $|a_8|\in[.12,.40]$, with independent random signs,
and $c_0\in[-.5,.5]$.  Each function is conjugated by a bounded
$g\in\mathrm{SL}(2,\mathbb R)$ generated using
$\theta\in[-\pi,\pi]$, log-scale in $[-.35,.35]$, and shear in $[-.5,.5]$.
This is deliberately not described as Haar sampling because
$\mathrm{SL}(2,\mathbb R)$ is noncompact.  Numerical integration uses a fixed
bounded sampling law on $[-.8,.8]^2$ and does not define the function's
symmetry.

\subsubsection{Local involutive-distribution benchmark}

Let $z=(z_1,\ldots,z_6)$ and $u=(z_3,\ldots,z_6)$.  In latent coordinates,
\begin{equation}
 F_0(z)=H(u),\qquad
 \ker J_{F_0}(z)=\operatorname{span}\{e_1,e_2\}.
\end{equation}
An instance-specific diffeomorphism is a composition of eight additive
coupling layers.  For a layer-dependent partition $(A,B)$,
\begin{align}
x_A&=z_A,\\
x_B&=z_B+a\odot\left[
\tanh(W_tz_A+b)+.55\sin(W_sz_A+\phi)
+.35\tanh(W_az_A)\odot\tanh(W_bz_A)
\right].
\label{eq:local_coupling_layer}
\end{align}
Every layer is inverted exactly by subtracting the bracketed shift.  The
amplitudes are sampled in $[.18,.32]$, the matrices are row-normalized random
matrices, and biases/phases are sampled from $[-.7,.7]$ and $[-\pi,\pi]$.

The transverse map is

\begin{align}
H(u)&=A_{\rm out} \left[A_{\rm core}\psi(A_{\rm in}u)+b_0\right],\\
\psi _j(t) &= t+\alpha_j\tanh( \omega_j t + b_j)
             +\eta _j\sin(\nu _j t+\varphi _j),
\end{align}
where all three $4\times4$ matrices have condition number at most two,
$\alpha_j\in[.08,.20]$, $\eta_j\in[-.08,.08]$,
$\omega_j,\nu_j\in[.7,1.4]$, $b_j\in[-.8,.8]$, and
$\varphi_j\in[-\pi,\pi]$.  The observed function and target distribution are
\begin{equation}
 F(x)=H\!\left(\pi_\perp\phi^{-1}(x)\right),
 \qquad
 \mathcal D_F(x)=D\phi_z\operatorname{span}\{e_1,e_2\},
 \quad z=\phi^{-1}(x).
 \label{eq:hard_local_family}
\end{equation}
Instances are rejected unless the diffeomorphism condition is at most 20, the
minimum transverse singular value is at least $.25$, the maximum coordinate
magnitude is at most 3, and the observable-Jacobian norm is at most 20.
Consequently the benchmark is difficult because of sparse local sampling and
spatial variation, not singular geometry.

\subsubsection{Physical orientation and superspin energies}

\paragraph{Quaternion orientation ($\mathrm{SO}(4)$).}
We sample $q\in S^3/\{q\sim-q\}$ and let $R(q)\in\mathrm{SO}(3)$ be the
standard homogeneous quadratic quaternion-to-rotation map.  After an
independent Haar orientation change in $\mathrm{SO}(4)$, the canonical energy
families are
\begin{align}
d=6:&\quad E(q)=c,\\
d\in\{2,3\}:&\quad E(q)=c+p_1t+p_2t^2+p_3t^3,
\quad t=R(q)_{33}~(d=2),~~t=\tfrac13\operatorname{tr}R(q)~(d=3),\\
d=1:&\quad E(q)=c+\sum_{j=1}^3p_jR(q)_{3j}^2,\\
d=0:&\quad E(q)=c+\sum_{j=1}^3\left[p_jR(q)_{jj}+.1p_{j+3}R(q)_{jj}^2\right].
\end{align}
For $d=2,3$, $p_1\in[1,1.5]$, $p_2\in[-.15,.15]$, and
$p_3\in[.05,.15]$.  For $d=1$,
$(p_1,p_2,p_3)=s(-1,u,1)$ with $s\in[.8,1.4]$ and
$u\in[-.2,.2]$.  For $d=0$,
$(p_1,p_2,p_3)=s(.35,.9,1.7)$ with $s\in[.8,1.2]$, and
$p_4,p_5,p_6\in[.5,1.5]$.

\paragraph{Five-dimensional superspin ($\mathrm{SO}(5)$).}
For $n\in S^4$,
\begin{equation}
 E(n)=c+t+b t^2,\qquad t=n^\top Bn,
 \label{eq:superspin_energy}
\end{equation}
where $B=R\operatorname{diag}(\lambda)R^\top$ and
$R\sim\operatorname{Haar}(\mathrm{SO}(5))$.  The eigenvalue multiplicities
are
\[
(1,1,1,1,1),~(2,1,1,1),~(2,2,1),~(3,1,1),~(3,2),~(5)
\]
for $d=0,1,2,3,4,10$, respectively.  Distinct eigenvalue centers are evenly
spaced on $[-.9,.9]$ and independently perturbed by at most $.05$;
$b\in[-.15,.15]$ and $c\in[-1,1]$.  Thus the stabilizer is determined exactly
by the eigenspace multiplicities of $B$, while the random orientation removes
any preferred coordinate system.

\label{app:full_results}
\subsection{$\mathrm{SO}(5)$ final-test results}
\label{app:so5_final_test}

We evaluate all methods on the same untouched final-test functions and on
nested observation prefixes
\(
N\in\{32,64,128,256,512,1024,2048\}
\).
The target is the orthogonal projector
\(P_{\mathfrak h_f}\in\mathbb R^{10\times 10}\) onto the symmetry Lie algebra
\(\mathfrak h_f\subset\mathfrak{so}(5)\), whose allowed dimensions are
\(d_f\in\{0,1,2,3,4,10\}\).
AEHE denotes the amortized equivariant harmonic estimator described in
Appendix~\ref{app:so5_details}.  AHC is the calibration-only harmonic model:
it estimates the harmonic coefficients independently for each function and
learns only invariant finite-sample calibration of the stabilizer blocks.
The analytic estimator uses the frozen analytic harmonic construction and its
analytic dimension rule.  Beyond-Affine is restricted to Killing fields in
\(\mathfrak{so}(5)\), so that its output is comparable to the same target
subspace.

\begin{table*}[t]
\centering
\caption{Untouched $\mathrm{SO}(5)$ final-test results.  ``Dim.'' is symmetry-
dimension accuracy, ``hard'' is
$\|\widehat P-P_{\mathfrak h_f}\|_F^2$, and ``oracle'' supplies the true
dimension only when converting the estimated stabilizer metric into a
projector.  Lower projector error is better.}
\label{tab:so5_full_final}
\small
\setlength{\tabcolsep}{4.0pt}
\begin{tabular}{c|ccc|cc|cc|ccc}
\toprule
& \multicolumn{3}{c|}{AEHE}
& \multicolumn{2}{c|}{AHC}
& \multicolumn{2}{c|}{Analytic harmonic}
& \multicolumn{3}{c}{Beyond-Affine} \\
$N$ & Dim. & Hard & Oracle & Dim. & Hard & Dim. & Hard
    & Dim. & Hard & Oracle \\
\midrule
32   & .668 & 1.265 & 1.107 & .514 & 1.737 & .111 & 7.440 & .130 & 2.497 & 1.530 \\
64   & .834 &  .794 &  .705 & .566 & 1.359 & .123 & 7.146 & .123 & 2.377 & 1.119 \\
128  & .931 &  .531 &  .500 & .646 & 1.022 & .201 & 4.299 & .214 & 1.890 &  .300 \\
256  & .973 &  .306 &  .295 & .746 &  .730 & .226 & 2.210 & .866 &  .159 &  .109 \\
512  & .989 &  .126 &  .117 & .807 &  .530 & .253 & 1.764 & .879 &  .131 &  .082 \\
1024 & .996 &  .030 &  .026 & .898 &  .298 & .334 & 1.296 & .891 &  .131 &  .082 \\
2048 & .998 &  .007 &  .005 & .943 &  .155 & .502 &  .827 & .891 &  .132 &  .074 \\
\bottomrule
\end{tabular}
\end{table*}

AEHE improves both parts of the inverse problem.  Relative to AHC, it shares
information across functions when estimating representation coefficients;
relative to the analytic and Beyond-Affine estimators, it also learns the
finite-sample dimension decision.  Beyond-Affine becomes geometrically strong
once its degree-four per-function polynomial fit is sufficiently determined,
as reflected by its oracle-dimension error, but its learned-free dimension rule
and hard-projector error plateau at large~$N$.  In contrast, AEHE continues to
improve through $N=2048$.  This separation supports the interpretation that
finite-sample representation estimation, rather than the analytic stabilizer
map itself, is the principal high-dimensional bottleneck.

For completeness, at $N=2048$ the three coupled one-dimensional subclasses
all attain unit symmetry-dimension accuracy.  For rate classes $1{:}1$,
$1{:}2$, and $2{:}3$, respectively, AEHE obtains projector errors
$0.00581$, $0.00449$, and $0.00238$; generator principal angles
$2.46^\circ$, $2.40^\circ$, and $1.87^\circ$; and recovered robust rate ratios
$1.018$, $2.114$, and $1.537$ (targets $1$, $2$, and $1.5$).  The angle and
projector errors are the primary diagnostics because the coordinate ratio is
less stable when a generator component is small.  The official LieGAN
comparison and the complete coupled-generator discussion remain in
Appendix~\ref{app:liegan_comparison}.

All model and threshold choices were frozen using training/development data.
Predictions for every final-test input were written and hashed before the
private final-test oracle was opened.  The final oracle was used only for the
reported scoring pass.
\subsection{$\mathrm{SO}(3)$ semi-supervised symmetry discovery}
\label{app:so3_semisupervised}

The mixed-symmetry $\mathrm{SO}(3)$ benchmark contains full, axial, and trivial
continuous-symmetry classes, including a designated hard subset whose
degree-two harmonic signal is exactly zero.  The physics target is computed
from $C=32$ oracle-gradient collocation points.  Physics-only training uses a
soft spectral projector obtained from the normalized physics Gram matrix; it
does not use the symmetry class, symmetry dimension, or target projector.
Semi-supervised training uses the exact frozen supervised objective on the
labeled subset and the spectral physics target on the remaining functions.
Nested class-balanced labeled subsets are fixed with subset seed 1729.
The physics coefficient $\lambda_{\mathrm{phys}}=3$ is selected once at
$L=512$ on validation data and then frozen for all label budgets.  We report
mean $\pm$ sample standard deviation over training seeds 1729, 271828, and
314159.

\begin{table*}[t]
\centering
\caption{$\mathrm{SO}(3)$ untouched final-test results.  Mean accuracy averages
the five observation resolutions.  The remaining columns use $N=512$.
``Hard'' denotes the no-$\ell=2$ subset.}
\label{tab:so3_semisupervised_results}
\scriptsize
\setlength{\tabcolsep}{3.5pt}
\begin{tabular}{lrrrrr}
\toprule
Configuration & Mean dim. acc. (\%) & Dim. acc. (\%) & Hard acc. (\%)
& Hard projector MSE & Soft projector MSE\\
\midrule
Supervised $L=32$   & $70.251\pm.240$ & $73.525\pm.874$ & $80.934\pm.405$ & $.291\pm.011$ & $.282\pm.012$\\
Semi $L=32$         & $85.519\pm.165$ & $97.680\pm.072$ & $99.215\pm.060$ & $.024\pm.001$ & $.018\pm.001$\\
Supervised $L=128$  & $77.194\pm.753$ & $86.163\pm1.974$& $92.233\pm.755$ & $.139\pm.020$ & $.118\pm.025$\\
Semi $L=128$        & $85.044\pm.426$ & $97.261\pm.370$ & $99.255\pm.201$ & $.028\pm.003$ & $.020\pm.002$\\
Supervised $L=512$  & $82.445\pm.235$ & $95.021\pm.430$ & $97.907\pm.759$ & $.051\pm.004$ & $.036\pm.004$\\
Trace-only $L=512$  & $82.629\pm.261$ & $95.394\pm.133$ & $98.068\pm.302$ & $.047\pm.001$ & $.035\pm.001$\\
Semi $L=512$        & $85.899\pm.064$ & $97.660\pm.120$ & $99.245\pm.060$ & $.024\pm.001$ & $.017\pm.001$\\
Supervised $L=3333$ & $85.046\pm.282$ & $96.981\pm.223$ & $99.054\pm.166$ & $.031\pm.002$ & $.022\pm.001$\\
Semi $L=3333$       & $85.916\pm.096$ & $97.847\pm.114$ & $99.366\pm.060$ & $.023\pm.001$ & $.017\pm.001$\\
Supervised $L=9999$ & $86.092\pm.061$ & $97.800\pm.053$ & $99.366\pm.080$ & $.023\pm.000$ & $.017\pm.001$\\
Physics-only $L=0$  & $86.243\pm.032$ & $97.940\pm.035$ & $99.366\pm.080$ & $.022\pm.000$ & $.016\pm.000$\\
\bottomrule
\end{tabular}
\end{table*}

At every matched explicit-label budget, the semi-supervised model improves
over supervised-only training.  Physics-only training, without symmetry
labels, matches or slightly exceeds the fully supervised $L=9999$ arm.  The
hard-subset result shows that the gain is not attributable to a degree-two
shortcut: the physics Gram supplies the correct low-eigenvalue subspace even
when that harmonic block vanishes.

\subsection{$\mathrm{SO}(4)$ analytic recovery and learned-calibration ablation}
\label{app:so4_additional}

The $\mathrm{SO}(4)$ benchmark uses symmetry dimensions
\(d_f\in\{0,1,2,3,6\}\), a fixed hyperspherical extractor, and stabilizer
blocks at degrees $\ell\in\{2,4,6\}$.  The learned model receives 58 invariant
features and predicts nonnegative scalar gates for the three blocks together
with a symmetry-dimension head.  It has $8{,}456$ trainable parameters and is
trained on 2,000 functions, validated on 500 functions, and evaluated over
three frozen seeds.  The untouched test set contains 2,500 balanced functions
(seed 20261002).

\begin{table}[t]
\centering
\caption{$\mathrm{SO}(4)$ symmetry-dimension accuracy for the analytic
$\ell=2$ estimator and the full learned-calibration model.}
\label{tab:so4_dim_accuracy}
\small
\begin{tabular}{c|rrrrrr}
\toprule
$N$ & 32 & 64 & 128 & 256 & 512 & 1024\\
\midrule
Analytic $\ell=2$ (\%) & 23.12 & 36.25 & 53.75 & 74.38 & 90.62 & 93.12\\
Full learned (\%)      & 79.73 & 89.23 & 96.39 & 99.16 & 99.89 & 99.99\\
\bottomrule
\end{tabular}
\end{table}

The frozen A/B/C/D mechanism ablation separates dimension inference from
subspace geometry: A uses analytic dimension and analytic $\ell=2$ geometry;
B uses the learned dimension head with analytic geometry; C uses analytic
dimension with learned calibrated geometry; and D uses both learned
components.  At $N=512$, replacing only the analytic dimension rule (A$\to$B)
raises dimension accuracy from $91.48\%$ to $99.89\%$.  Independently, on the
rank-three class and with the oracle dimension supplied, replacing analytic
geometry by the learned calibrated metric reduces projector error from
$0.2794$ to $0.1181$ and the largest principal angle from $9.98^\circ$ to
$3.71^\circ$.  The full model therefore needs both components.  At $N=1024$
it obtains hard/oracle-dimension projector errors $0.0434/0.0432$ and mean
largest principal angle $1.52^\circ$.

The gates give the expected mechanism.  Within rank three, the weakest
degree-two-SNR quartile uses mean
\((w_2,w_4,w_6)=(0.2112,0.7218,0.0671)\), whereas the strongest quartile uses
\((0.5587,0.4396,0.0017)\).  Thus the learned metric recruits degree four when
degree two is unreliable.  Gate values for the fully symmetric rank-six class
are not scientifically interpreted because all population anisotropic blocks
vanish and trace-normalized finite-sample noise can make their relative weights
arbitrary.

\subsection{Physical applications: protocol and full sample/noise sweeps}
\label{app:physical_full_sweeps}

We use two simulated energy landscapes.  The quaternion experiment represents
orientations by $q\in S^3/\{q\sim-q\}$ and produces symmetry dimensions
\(\{0,1,2,3,6\}\) under Haar-random global $\mathrm{SO}(4)$ orientations.
The superspin experiment uses $n\in S^4$ and energies
\[
 E(n)=c+t+b t^2,\qquad t=n^\top Bn,
\]
where eigenvalue multiplicities of $B$ generate symmetry dimensions
\(\{0,1,2,3,4,10\}\); each instance is Haar-randomly oriented in
$\mathrm{SO}(5)$.  Each class uses 128/32/64 independent
train/validation/test functions.  We evaluate nested prefixes
\(N\in\{32,64,128,256,512,1024\}\) with additive output noise
\(\sigma\in\{0,0.01,0.05\}\) in fixed energy units and training seeds 1729,
271828, and 314159.

The amortized estimator (AEHE) uses the frozen coefficient-learning mechanism,
with sphere-mean subtraction replacing radial least squares.  The
calibration-only estimator (AHC) fits coefficients independently and learns
only invariant stabilizer calibration.  Analytic LS performs both coefficient
estimation and stabilizer recovery independently per function.  The
Killing-restricted Beyond-Affine baseline fits a polynomial surrogate per
function and restricts the discovered affine fields to the appropriate compact
Lie algebra.  All degrees, ridge parameters, and spectral thresholds are
selected on validation functions only.

\begin{table*}[t]
\centering
\caption{$\mathrm{SO}(4)$ quaternion-energy sweep.  Each method reports
symmetry-dimension accuracy / hard projector error.}
\label{tab:physical_so4_full}
\scriptsize
\setlength{\tabcolsep}{4pt}
\begin{tabular}{cc|rrrr}
\toprule
$\sigma$ & $N$ & AEHE & AHC & Analytic LS & Beyond-Affine\\
\midrule
0 & 32   & .875/.5806 & .942/.4119 & .263/2.248 & .672/.3638\\
0 & 64   & .934/.3709 & .991/.1378 & 1.000/$4.194\!\times\!10^{-5}$ & .797/.2036\\
0 & 128  & .953/.2711 & 1.000/.03202 & 1.000/$7.544\!\times\!10^{-6}$ & .994/.007439\\
0 & 256  & .956/.2193 & 1.000/.01308 & 1.000/$2.594\!\times\!10^{-6}$ & 1.000/.0001293\\
0 & 512  & .959/.1958 & 1.000/.004843& 1.000/$1.174\!\times\!10^{-6}$ & 1.000/$4.472\!\times\!10^{-5}$\\
0 & 1024 & .961/.1787 & 1.000/.002411& 1.000/$4.897\!\times\!10^{-7}$ & 1.000/$2.087\!\times\!10^{-5}$\\
\midrule
.01 & 32   & .876/.5805 & .938/.4231 & .250/2.291 & .459/1.580\\
.01 & 64   & .936/.3697 & .989/.1438 & .800/1.200 & .597/1.404\\
.01 & 128  & .956/.2688 & 1.000/.03201& .800/1.200 & .794/1.207\\
.01 & 256  & .957/.2179 & 1.000/.01311& .800/1.200 & .800/1.200\\
.01 & 512  & .959/.1942 & 1.000/.004836& .800/1.200 & .800/1.200\\
.01 & 1024 & .960/.1796 & 1.000/.002413& .800/1.200 & .800/1.200\\
\midrule
.05 & 32   & .869/.5980 & .948/.4171 & .200/2.400 & .287/1.886\\
.05 & 64   & .930/.3922 & .986/.1480 & .697/1.326 & .469/1.597\\
.05 & 128  & .953/.2809 & 1.000/.03226& .797/1.203 & .775/1.227\\
.05 & 256  & .949/.2411 & 1.000/.01336& .800/1.200 & .800/1.200\\
.05 & 512  & .957/.2145 & 1.000/.004856& .800/1.200 & .800/1.200\\
.05 & 1024 & .954/.2042 & 1.000/.002442& .800/1.200 & .800/1.200\\
\bottomrule
\end{tabular}
\end{table*}

\begin{table*}[t]
\centering
\caption{$\mathrm{SO}(5)$ superspin-energy sweep.  Each method reports
symmetry-dimension accuracy / hard projector error.}
\label{tab:physical_so5_full}
\scriptsize
\setlength{\tabcolsep}{4pt}
\begin{tabular}{cc|rrrr}
\toprule
$\sigma$ & $N$ & AEHE & AHC & Analytic LS & Beyond-Affine\\
\midrule
0 & 32   & .614/1.382 & .537/1.403 & .167/3.333 & .586/.6798\\
0 & 64   & .701/.8294 & .714/.7667 & .172/3.285 & .956/.05366\\
0 & 128  & .791/.4232 & .875/.3199 & 1.000/$9.181\!\times\!10^{-16}$ & .997/.002937\\
0 & 256  & .839/.2439 & .951/.1023 & 1.000/$2.237\!\times\!10^{-16}$ & 1.000/$3.226\!\times\!10^{-5}$\\
0 & 512  & .858/.1968 & .991/.03344& 1.000/$9.118\!\times\!10^{-17}$ & 1.000/$8.794\!\times\!10^{-6}$\\
0 & 1024 & .892/.1348 & .998/.01261& 1.000/$4.037\!\times\!10^{-17}$ & 1.000/$3.710\!\times\!10^{-6}$\\
\midrule
.01 & 32   & .603/1.401 & .544/1.409 & .167/3.333 & .417/2.326\\
.01 & 64   & .712/.8071 & .727/.7769 & .167/3.313 & .789/1.722\\
.01 & 128  & .787/.4297 & .872/.3225 & .833/1.667 & .831/1.670\\
.01 & 256  & .846/.2384 & .952/.09882& .833/1.667 & .833/1.667\\
.01 & 512  & .857/.1993 & .989/.03359& .833/1.667 & .833/1.667\\
.01 & 1024 & .891/.1365 & .997/.01350& .833/1.667 & .833/1.667\\
\midrule
.05 & 32   & .586/1.410 & .546/1.452 & .167/3.333 & .318/2.563\\
.05 & 64   & .663/.8744 & .707/.8004 & .167/3.333 & .656/1.911\\
.05 & 128  & .749/.5008 & .857/.3175 & .643/1.900 & .820/1.694\\
.05 & 256  & .801/.3207 & .951/.1096 & .833/1.668 & .833/1.670\\
.05 & 512  & .816/.2638 & .987/.04018& .833/1.667 & .833/1.668\\
.05 & 1024 & .812/.2579 & .999/.01204& .833/1.667 & .833/1.667\\
\bottomrule
\end{tabular}
\end{table*}

In the clean regime, independently fitted harmonic coefficients can be nearly
exact once $N$ is large enough.  Under noise, however, the independent
coefficient and polynomial fits often collapse to an incorrect stable
dimension decision, producing the characteristic plateaus in Tables
\ref{tab:physical_so4_full}--\ref{tab:physical_so5_full}.  AHC is the strongest
noise-robust method in these two controlled physical simulations: it combines
analytic coefficient estimation with learned cross-function spectral
calibration.  AEHE is more robust than raw analytic LS in many sparse/noisy
regimes, but its coefficient amortization introduces bias and does not dominate
AHC here.  Consequently, these applications support amortization as a
noise-stabilization mechanism, not a universal replacement for accurate
per-function least squares.

\subsection{Metrics, model selection, and reproducibility protocol}
\label{app:evaluation_protocol}

For an estimated symmetry dimension $\widehat d_f$ and hard projector
\(\widehat P_f\), we report
\[
 \operatorname{Acc}_{\rm dim}
 =\frac{1}{|
 \mathcal T|}\sum_{f\in\mathcal T}
 \mathbf 1\{\widehat d_f=d_f\},
 \qquad
 L_{\rm proj}=\|\widehat P_f-P_{\mathfrak h_f}\|_F^2.
\]
When a method produces a soft projector, its squared Frobenius error is reported
separately.  Oracle-dimension error uses the method's estimated metric or
geometry but supplies $d_f$ only at the final eigenspace-extraction step; it
therefore diagnoses geometry independently of dimension classification.  For
equal-dimensional subspaces with orthonormal bases $U$ and $V$, principal
angles are obtained from the singular values of $U^\top V$; we report the
largest angle, and the mean angle where used consistently by the corresponding
experiment.

All benchmark functions are separated at the function-instance level across
training, validation, and test splits.  Nested observation prefixes reuse the
same maximum observation pool, preventing changes in the sampled functions
from being confounded with changes in $N$.  Hyperparameters, spectral
thresholds, polynomial degrees, ridge values, early stopping, and checkpoint
selection use training/development or validation data only.  For frozen final
tests, predictions are generated from public sample arrays and hashed before
the private oracle archive is opened.  Final-test labels are then loaded once
for scoring; no model, threshold, architecture, or selection rule is changed
afterward.  Reported learned-method uncertainty is the sample standard
deviation over the frozen training seeds, while deterministic analytic
baselines are evaluated once on the identical functions, queries, and nested
prefixes.

\subsection{Local symmetry recovery: implementation details and complete comparison}
\label{app:local_symmetry_details}

The local benchmark consists of maps
\(F:\mathbb R^6\rightarrow\mathbb R^4\) with the rank-two involutive
distribution
\(
\mathcal D_F(x)=\ker J_F(x)
\).
Train, validation, and test splits contain $5000$, $512$, and $512$
independently generated functions.  Evaluation uses eight independently
sampled queries per validation/test function and nested observation prefixes
\(N\in\{64,128,256,512,1024,2048\}\).

\paragraph{Analytic local estimator.}
The analytic estimator performs weighted second-order local polynomial
regression around each query.  With normalized displacement
\(\delta_i=x_i-x_q\), its design contains
\([1,\delta_i,\operatorname{vech}(\delta_i\delta_i^\top)]\), i.e., 28
features in six input dimensions.  A fixed Gaussian neighborhood and a small
ridge give an analytic estimate of $J_F(x_q)$; the two-dimensional right
nullspace gives the predicted projector.

\paragraph{Amortized Local Weighting (ALW).}
ALW preserves the same polynomial regression, analytic linear solve, Jacobian
extraction, and nullspace calculation, but replaces the fixed kernel by learned
query-dependent weights.  The weight network receives normalized offsets,
radius, output differences, and cloud-level invariant summaries (mean,
variance, and skewness), together with a global DeepSet mean/max summary of
$(x_i,F(x_i))$.  A three-layer scalar MLP produces logits $s_i$ and
\[
  w_i=0.98\,\operatorname{softmax}(s)_i+\frac{0.02}{K}.
\]
The selected model uses $K=96$ neighbors and has $41{,}089$ trainable
parameters.  It is trained with AdamW (learning rate $5\times10^{-4}$, weight
decay $10^{-4}$, batch size 64, at most 30 epochs), using seed 1730.  Selection
uses validation projector error at $N=128$ and $N=512$; early stopping starts
at epoch 10 with patience 7 and $\texttt{min\_delta}=10^{-5}$.

\paragraph{Beyond-Affine.}
The closest value-only adaptation first fits a total-degree polynomial
surrogate to each sampled function and then optimizes polynomial vector fields
whose directional derivatives annihilate that surrogate.  The two coefficient
vectors in the bottom eigenspace are evaluated at the query and orthonormalized
to form a rank-two tangent projector.  Surrogate degree
\(\{2,3\}\), ridge \(\{10^{-6},10^{-4}\}\), vector-field degree
\(\{0,1,2,3\}\), and generator regularization
\(\{0,10^{-6},10^{-4}\}\) are selected using 128 validation functions only.
This is a conservative comparison: Beyond-Affine parameterizes polynomial
annihilating fields, whereas the target here is a smoothly varying involutive
rank-two distribution.

\paragraph{Direct neural projector baseline.}
The appendix baseline directly predicts a $6\times2$ frame from a pooled
set/query representation and converts it to a rank-two projector by QR
orthogonalization.  The selected global-plus-value DeepSet has approximately
$102{,}540$ parameters and is trained with AdamW (learning rate $10^{-3}$,
weight decay $10^{-4}$, batch size 64, at most 30 epochs).  Unlike ALW, it does
not retain analytic local differentiation or nullspace recovery.

\begin{table*}[t]
\centering
\caption{Local-symmetry final-test comparison.  Projector error is the mean
squared Frobenius error; angle is the median largest principal angle in
degrees.  Lower is better.}
\label{tab:local_full_comparison}
\small
\setlength{\tabcolsep}{5.0pt}
\begin{tabular}{c|rrrr|rrrr}
\toprule
& \multicolumn{4}{c|}{Mean projector error}
& \multicolumn{4}{c}{Median largest angle ($^\circ$)}\\
$N$ & ALW & AL & BA & Direct neural
    & ALW & AL & BA & Direct neural\\
\midrule
64   & .01608 & .02055 & .02132 & .26016 & 3.65 & 4.12 & 4.21 & 17.77\\
128  & .00856 & .01017 & .01717 & .21298 & 2.63 & 2.93 & 3.73 & 15.84\\
256  & .00490 & .00617 & .01588 & .18710 & 2.02 & 2.30 & 3.53 & 14.77\\
512  & .00293 & .00382 & .01529 & .17501 & 1.57 & 1.85 & 3.45 & 14.32\\
1024 & .00177 & .00233 & .01501 & .17143 & 1.24 & 1.46 & 3.40 & 14.24\\
2048 & .00107 & .00145 & .01486 & .18107 &  .97 & 1.14 & 3.38 & 14.53\\
\bottomrule
\end{tabular}
\end{table*}

ALW improves consistently over the fixed analytic kernel while retaining its
interpretable differentiation and nullspace stages.  The direct projector
network is substantially worse and saturates, indicating that amortizing the
entire geometric map is less data-efficient than learning only the local
weighting required by a correct analytic estimator.  Beyond-Affine is
competitive at the smallest prefix but plateaus because a fixed-degree global
polynomial vector field is mismatched to the continuously varying local
distribution.

\subsection{$\SL(2,\mathbb R)$ finite-sample experiment}
\label{app:sl2_experiment}

We test whether the stabilizer construction depends on compact orthogonal
representations. The benchmark contains elliptic, hyperbolic, parabolic, and
trivial stabilizers under bounded random conjugations. Functions are represented
in a finite polynomial space, and the corresponding infinitesimal
representation replaces the hyperspherical-harmonic construction used for
$\SO(n)$.

The common polynomial representation contains 28 coefficients, estimated by
ordinary least squares. At $N=16$ the system is underdetermined and
symmetry-dimension accuracy is $31.45\%$. At $N=32$, both symmetry-dimension
and conjugacy-type accuracy reach $100\%$, with projector error below
$2\times10^{-24}$. Population coefficient and infinitesimal Grams recover all
target stabilizers to numerical precision.

This experiment provides two controls. First, the representation-theoretic
stabilizer construction extends beyond compact $\SO(n)$ actions. Second,
learning is not required when the finite-sample representation can already be
recovered reliably.

\subsection{LieGAN comparison and coupled-generator recovery}
\label{app:liegan_comparison}

\paragraph{Benchmark.}
We compare against the official LieGAN implementation of
\citet{yang2023generative} on the three one-dimensional coupled-rotation
subclasses
\[
\mathfrak h_{1:1}
=
\operatorname{span}\{E_{12}+E_{34}\},\qquad
\mathfrak h_{1:2}
=
\operatorname{span}\{E_{12}+2E_{34}\},
\]
\[
\mathfrak h_{2:3}
=
\operatorname{span}\{2E_{12}+3E_{34}\}.
\]
The corresponding canonical relative rotation rates are
$1$, $2$, and $3/2$.
These classes require recovering not only a one-dimensional symmetry algebra,
but also the relative weights of the two coupled planar rotations.

LieGAN is evaluated only on these one-dimensional classes and is configured
to learn one generator per function. It therefore does not predict symmetry
dimension and is not included in the full symmetry-dimension comparison.
We instead evaluate generator angle, rank-one projector error, off-plane
leakage, relative rotation rate, infinitesimal residual, and runtime.

\paragraph{LieGAN implementation.}
We use the authors' (\cite{yang2023generative}) public implementation.
For samples
$\mathcal D_f=\{(x_i,f(x_i))\}_{i=1}^{N}$,
the input $x_i\in\mathbb R^5$ is treated as a geometric vector and
$f(x_i)$ as an invariant scalar. A learned generator $L$ acts through
\[
x_i'=\exp(zL)x_i,\qquad
y_i'=f(x_i),\qquad
z\sim\mathcal N(0,1),
\]
and a discriminator is trained to distinguish
$(x_i,f(x_i))$ from $(x_i',f(x_i))$.

We evaluate two generator parameterizations:
(i) an unconstrained $L\in\mathbb R^{5\times5}$, and
(ii) an $\SO(5)$-restricted parameterization
$L=\frac12(B-B^\top)\in\mathfrak{so}(5)$.
For the unconstrained model, the learned matrix is projected onto its
skew-symmetric component before evaluating $\mathfrak{so}(5)$ generator
geometry. The restricted variant therefore provides LieGAN with the correct
ambient transformation family.

Hyperparameters are selected on independent validation functions without
access to ground-truth generators or projectors. Both variants select
\[
\lambda_{\mathrm{cos}}=1,\qquad
\eta_G=10^{-3},\qquad
\eta_D=2\times10^{-4},
\]
with batch size $64$, $z\sim\mathcal N(0,1)$, and $100$ epochs.
Each optimization is repeated with seeds $101$, $202$, and $303$.
We additionally verify on controlled examples that the implementation recovers
simple rotations, respects generator sign and scale in the angle metric, and
assigns near-zero infinitesimal residual to the true generator.

\paragraph{Matched pilot.}
The paired pilot contains $12$ independently generated functions from each
coupled subclass, for $36$ functions in total, evaluated at
$N\in\{512,1024,2048\}$.
LieGAN is optimized independently for every function, sample count,
parameterization, and optimization seed.
The paired comparator is the frozen structured harmonic-calibration model
from Experiment~3C.
The later AEHE model is evaluated separately on its own untouched final test
and is therefore not part of this paired comparison.

\begin{table}[t]
\centering
\caption{
Matched coupled-class pilot, averaged over the three subclasses and three
optimization seeds or frozen checkpoints.
Projector error is squared Frobenius error.
}
\label{tab:liegan_pilot}
\small
\setlength{\tabcolsep}{4pt}
\begin{tabular}{r l cc}
\toprule
$N$ & Method & Mean angle $\downarrow$ & Proj.\ err.\ $\downarrow$ \\
\midrule
512
& Harmonic calibration & $36.78^\circ$ & .794 \\
& LieGAN, general      & $71.66^\circ$ & 1.307 \\
& LieGAN, $\SO(5)$     & $71.83^\circ$ & 1.307 \\
\midrule
1024
& Harmonic calibration & $34.50^\circ$ & .745 \\
& LieGAN, general      & $68.82^\circ$ & 1.269 \\
& LieGAN, $\SO(5)$     & $68.91^\circ$ & 1.266 \\
\midrule
2048
& Harmonic calibration & $21.68^\circ$ & .483 \\
& LieGAN, general      & $65.16^\circ$ & 1.217 \\
& LieGAN, $\SO(5)$     & $65.61^\circ$ & 1.217 \\
\bottomrule
\end{tabular}
\end{table}

The harmonic comparator improves steadily with sample count, whereas both
LieGAN variants remain at large generator angle.
Restricting the generator to $\mathfrak{so}(5)$ removes symmetric-matrix
degrees of freedom but does not resolve the remaining optimization ambiguity
within the rotational generator space.

\paragraph{Per-class diagnostics.}
Table~\ref{tab:liegan_per_class} reports the matched $N=2048$ results.
Generator angle and projector error are the primary geometric metrics; the
rate ratio is more sensitive because it divides by an estimated coordinate.

\begin{table}[t]
\centering
\caption{
Per-class matched results at $N=2048$.
Angle, projector error, and off-plane leakage are means; rate ratio is the
median. Ground-truth rate ratios are $1$, $2$, and $1.5$.
}
\label{tab:liegan_per_class}
\small
\setlength{\tabcolsep}{3.5pt}
\begin{tabular}{c l c c c c}
\toprule
Class & Method
& Angle $\downarrow$
& Proj.\ $\downarrow$
& Off-plane $\downarrow$
& Ratio \\
\midrule
$1{:}1$
& Harmonic cal. & $36.96^\circ$ & .769 & .517 & 1.224 \\
& LieGAN, gen.  & $63.40^\circ$ & 1.189 & .796 & .725 \\
& LieGAN, $\SO(5)$ & $66.77^\circ$ & 1.228 & .814 & .833 \\
\midrule
$1{:}2$
& Harmonic cal. & $14.25^\circ$ & .346 & .075 & 4.190 \\
& LieGAN, gen.  & $64.82^\circ$ & 1.220 & .654 & .836 \\
& LieGAN, $\SO(5)$ & $64.99^\circ$ & 1.224 & .666 & .909 \\
\midrule
$2{:}3$
& Harmonic cal. & $13.84^\circ$ & .335 & .114 & 2.184 \\
& LieGAN, gen.  & $67.22^\circ$ & 1.242 & .725 & .846 \\
& LieGAN, $\SO(5)$ & $65.08^\circ$ & 1.199 & .655 & .848 \\
\bottomrule
\end{tabular}
\end{table}

\paragraph{AEHE coupled-class recovery.}
AEHE is evaluated separately on the untouched 3E final test, containing
$224$ functions from each coupled subclass.
These results are not paired with the LieGAN pilot and should be interpreted
as a final diagnostic of the amortized estimator rather than as a direct
instance-wise comparison.

\begin{table}[t]
\centering
\caption{
AEHE coupled-class recovery on the untouched 3E final test at $N=2048$.
Projector error is squared Frobenius error.
}
\label{tab:aehe_coupled}
\small
\setlength{\tabcolsep}{4pt}
\begin{tabular}{c c c c c}
\toprule
Class
& Dim.\ acc.
& Proj.\ err.\ $\downarrow$
& Angle $\downarrow$
& Median ratio \\
\midrule
$1{:}1$ & 1.000 & .00581 & $2.46^\circ$ & 1.018 \\
$1{:}2$ & 1.000 & .00449 & $2.40^\circ$ & 2.114 \\
$2{:}3$ & 1.000 & .00238 & $1.87^\circ$ & 1.537 \\
\bottomrule
\end{tabular}
\end{table}

AEHE recovers all three coupled classes with perfect symmetry-dimension
accuracy, small projector error, and generator angles below $2.5^\circ$.
The recovered relative rotation rates are also close to their canonical values.

\paragraph{Computation.}
LieGAN requires a separate adversarial optimization for every function,
sample count, and random seed. On an NVIDIA A6000, one optimization takes
approximately $18$ s at $N=512$, $30$-$36$ s at $N=1024$, and
$59$-$72$ s at $N=2048$.
By contrast, the harmonic-calibration and AEHE models reuse frozen parameters
across functions. In the reference 3E implementation, batched AEHE inference
over $1{,}008$ functions takes approximately $0.55$--$0.74$ s for
$N=64$-$2048$, excluding one-time data loading. The computational distinction follows directly from the inference strategy:  LieGAN performs test-time optimization for each function, whereas AEHE
amortizes the map $\mathcal D_f\longmapsto\widehat{\mathfrak h}_f$ over the function family.